\RequirePackage{fix-cm}
\documentclass[11pt]{article}
\usepackage[letterpaper,margin=1in]{geometry}
\usepackage[T1]{fontenc}
\usepackage{lmodern}
\usepackage[authoryear,round]{natbib}
\setcitestyle{citesep={;},aysep={,},yysep={;}}
\usepackage{amsmath,amssymb,amsthm,mathtools,mathrsfs}
\DeclareFontFamily{U}{rsfs}{\skewchar\font127}
\DeclareFontShape{U}{rsfs}{m}{n}{<-6.5> rsfs5 <6.5-8> rsfs7 <8-> rsfs10}{}
\usepackage{enumitem,graphicx,wrapfig,longtable,booktabs,tabularx,microtype,needspace}
\usepackage[colorlinks=true,allcolors=blue]{hyperref}
\usepackage{macros}
\usepackage{comment}
\usepackage{caption}
\usepackage{pgfplots,import,placeins,xfrac,xurl}
\pgfplotsset{compat=1.18}
\graphicspath{{figures/}}
\usepackage{url}

\title{Global Exponential Convergence\\of Two-Layer Linear Network Training}
\author{
  Stephen Zhang\\
  {\normalsize Flatiron Institute \& NYU} \\ 
  {\small\texttt{syz@syz.id.au}}
  \and
  Gabriel Peyr\'e\\
  {\normalsize CNRS, ENS, PSL Univ.}\\
  {\small\texttt{gabriel.peyre@ens.fr}}
}
\date{}
\hypersetup{
  pdftitle={Global Exponential Convergence of Two-Layer Linear Network Training},
  pdfauthor={Stephen Zhang and Gabriel Peyr\'e}
}

\begin{document}
\maketitle

\begin{abstract}
We prove global exponential (linear) convergence with an explicit rate in the rich scaling for wide two-layer linear networks trained with smooth
Polyak--\L{}ojasiewicz predictor losses. Gradient flow in the factors closes exactly in terms of a finite-dimensional Bures flow of the neuron law covariance, in which the predictor dynamics are preconditioned by hidden covariance blocks.
Mean-field conservation laws provide uniform spectral lower bounds on the hidden preconditioning blocks when the initial covariance satisfies a spectral support gap condition.
This condition encompasses positive definiteness while still allowing for singular initializations.
For an initial covariance $\Sigma_0 = \sigma^2 \Id$, the loss converges to the global minimum with linear rate at least $4\sigma^2\kappa$, where $\kappa$ is the PL constant. We establish stability of this rate under finite-width sampling, as well as global convergence of factor gradient descent for an explicit stepsize interval depending on smoothness, the initial loss, and conserved spectral margins.
Our argument extends layerwise to deep linear ResNets, subject to a residual-path bound.
In the case of heavy-ball momentum, training dynamics close instead over positions and velocities in terms of a lifted phase covariance.
Linear convergence holds under an explicit condition on the energy and damping, specifying a window of admissible dampings. For two-scale white initializations, this interval is nonempty for sufficiently large position scales, with a fixed initial loss gap and velocity covariance. Numerical experiments illustrate the covariance geometry and compare the predicted and observed rates.


\end{abstract}


\section{Introduction}
\label{sec:intro}

Two-layer linear networks are nonconvex in their weight factors, even when the
loss is a simple function of the predictor matrix.  Training moves in the
lifted space of incoming and outgoing neuron weights, whose second moments act
as a time-dependent preconditioner on the predictor.  These moments may in
principle become degenerate and slow down the dynamics, and it is not clear that the dynamics necessarily converge to a globally optimal predictor.
We consider this question in the rich mean-field scaling, in which both sides of each rank-one neuron evolve, and
smooth predictor losses satisfying a Polyak--\L{}ojasiewicz (PL) inequality, which subsume the standard quadratic loss.

Central to our analysis are width-independent conservation laws for
the weight covariance: the signed covariance spectrum, relative to the input-output
block split, supplies conserved supporting inequalities and lower bounds on each hidden-block. The sum of these lower bounds conditions the predictor dynamics: a positive sum
is sufficient to guarantee global convergence of the predictor loss with an explicit rate. This analysis covers every positive-definite initialization and some singular ones.
For initial variance $\sigma^2 \Id$, the loss decays with rate at least $4\sigma^2\kappa$, where $\kappa$ is the PL constant.

The covariance viewpoint also makes the result robust to finite-width
sampling and factor gradient descent. For deep linear ResNets, the same
conservation argument applied layerwise yields linear convergence
conditional on a uniform bound on the residual path. With heavy-ball
momentum, closure is lifted to the joint phase covariance of positions
and velocities. Conformally transported spectral laws describe its
geometry, while energy dissipation gives an explicit initialization and
damping certificate for linear convergence. Special values of the
conserved quantities further permit lower-dimensional intrinsic dynamics;
these reductions are described in Appendix~\ref{app:intrinsic}.

\paragraph{Related work}
Mean-field analyses describe the collective evolution of neuron weight
distributions in wide two-layer networks~\citep{ChizatBach2018,MeiMontanariNguyen2018,RotskoffVandenEijnden2018}.
These formulations treat broad nonlinear feature maps and often prove asymptotic global optimality or propagation-of-chaos type statements under support, homogeneity, noise, or regularity assumptions.
In the present linear setting, exact second-moment closure connects
training to Bures--Wasserstein covariance geometry~\citep{Takatsu2011,BhatiaJainLim2019},
making a sharp quantitative linear convergence rate available.
This closure distinguishes our analysis from the NTK~\citep{JacotGabrielHongler2018} and lazy training~\citep{ChizatOyallonBach2019} regime: both incoming and outgoing weights move, and the evolving representation is essential to the rate.

Deep linear networks have served as a tractable model of the nonconvexity created by matrix products, with results on their landscapes, exact dynamics,
and convergence~\citep{SaxeMcClellandGanguli2014,Kawaguchi2016,LaurentVonBrecht2018,AroraCohenGolowichHu2019}.
Finite-width analyses of two-layer matrix factorization furnish convergence guarantees under small, asymmetric, or general initializations~\citep{YeDu2021,XuMinTarmounMalladaVidal2023}, while recent work treats discrete-time low-rank adaptation directly~\citep{MuKlabjan2025}.
Our result instead starts from a population covariance satisfying a positive \emph{support gap} condition, in particular any positive-definite covariance, and obtains an explicit global rate that survives independent finite-width sampling.
Related work identifies Gramian differences and balancedness as invariants of factor training~\citep{DuHuLee2018,TarmounFrancaHaeffeleVidal2021} and develops
a general differential framework for identifying conservation laws~\citep{MarcotteGribonvalPeyre2023}.
Our mean-field laws are functionals of the neuron distribution that induce
compatible particle laws at every width. The conserved spectrum of the
signed covariance yields hidden-block lower bounds without requiring
a special balancedness identity.

Residual networks~\citep{HeZhangRenSun2016Deep,HeZhangRenSun2016Identity} were introduced as architectures built from perturbations of the identity, and their infinite depth limit leads naturally to neural ODEs~\citep{HeZhangRenSun2016Deep,ChenRubanovaBettencourtDuvenaud2018}.
Mean-field formulations have established global optimality and convergence results under expressivity, regularity, or local PL conditions~\citep{LuMaLuLuYing2020,DingChenLiWright2021,Isobe2023,BarboniPeyreVialard2024COTResNets}.
Our linear analysis provides a complementary condition: uniform conditioning of the forward and backward residual products is sufficient for convergence.

Momentum turns first-order descent into an inertial, second-order iteration; the heavy-ball method of~\cite{Polyak1964} and Nesterov acceleration~\citep{Nesterov1983} have mature convergence theories, including continuous-time interpretations~\citep{SuBoydCandes2016} and robust-control analyses~\citep{LessardRechtPackard2016}.
For deep linear networks, \citet{WangLinAbernethy2021,LiuTaoPan2022}
establish acceleration results under initialization and overparameterization
conditions. Here the appropriate state variable is the phase covariance, over which the dynamics close.
We give an explicit energy-damping certificate for global convergence
with general position-velocity correlations. 

\paragraph{Our contributions}
Section~\ref{sec:laws} identifies width-independent spectral conservation
laws and derives uniform hidden-block bounds from their support margins
(Propositions~\ref{prop:covariance-closure-laws} and~\ref{lem:margins}).
Their Wasserstein interpretation and local completeness among smooth covariance
laws common to all split losses are developed in
Appendices~\ref{app:mean-field} and~\ref{app:law-counting}.
Section~\ref{sec:convergence} proves global exponential convergence with an explicit rate under a smooth PL predictor loss from every initial covariance with positive
support gap (Theorem~\ref{thm:gf}).
The rate survives independent finite-width sampling at width proportional to the population covariance rank, which can be
as small as $\min(d_u,d_v)$ (Corollary~\ref{cor:finite-width} and
Remark~\ref{rem:rank-adaptive-width}).
We establish global linear convergence of factor
gradient descent under an explicit stepsize bound by controlling the
accumulated conservation error (Theorem~\ref{thm:gd}).
Section~\ref{sec:resnets} extends the layerwise argument to finite- and
continuous-depth linear ResNets under residual-path control
(Theorems~\ref{thm:resnet-finite} and~\ref{thm:resnet-continuous}), and
discusses why that control is not automatic (Remark~\ref{rem:resnet-path-escape}
in Appendix~\ref{app:resnets}).
Section~\ref{sec:momentum} shows that heavy-ball momentum dynamics closes on the phase covariance and identifies conformal conservation laws (Proposition~\ref{prop:phase-laws}).
Global exponential convergence is proven under an initial-energy
and damping condition, including a two-scale white specialization
(Theorem~\ref{thm:momentum} and Corollary~\ref{cor:two-scale-white-window}).
A computational toolbox, code reproducing all six article figures, and
illustrative Jupyter notebooks are available at
\url{https://github.com/gpeyre/linear-networks-global}.



\paragraph{Setup}\label{sec:model}
Writing training in terms of the predictor separates the loss geometry
from the nonconvex factorization.
A width-$n$ two-layer linear network has factors
$U:=(u_1,\ldots,u_n)\in\R^{d_u\times n}$ and
$V:=(v_1,\ldots,v_n)\in\R^{d_v\times n}$.
Write $m:=d_u+d_v$ and $W:=\binom UV$. Its predictor is
$\Svu:=VU^\top/n$, and its factor objective is
$\cL(U,V):=\ell(VU^\top/n)$ for a differentiable predictor loss
$\ell:\R^{d_v\times d_u}\to\R$.
For example, if a data distribution $\rho$ on pairs $(z,y)$ satisfies
$y=\Svu_\star z$ for a teacher $\Svu_\star$, squared regression loss is
\[
  \textstyle \ell_{\rm quad}(\Svu):=\frac12\int\norm{(\Svu-\Svu_\star)z}^2\,\mathrm{d}\rho(z,y).
\]
More generally, a self-adjoint positive semidefinite operator
$\quadH$ on $\R^{d_v\times d_u}$ defines a quadratic loss
$\ell(S):=\frac12\ip{\quadH[S-\Svu_\star]}{S-\Svu_\star}$.
The regression example corresponds to $\quadH[X]=XC_z$, where
$C_z:=\int zz^\top\,\dd\rho(z,y)$ is the input second moment.

We state the assumptions on the predictor loss: assume throughout that \(\ell:\R^{d_v\times d_u}\to\R\) is \(C^1\) and bounded below, and write \(\ell_\star:=\inf_S\ell(S)\).

\begin{assumption}[Smooth PL predictor loss]
\label{ass:predictor-loss}\label{ass:loss}
There exist constants \(L>0\) and \(\kappa>0\) such that the loss
\(\ell\) is \(L\)-smooth and satisfies the \(\kappa\)-Polyak--\L{}ojasiewicz
inequality: \(\norm{\nabla\ell(S)-\nabla\ell(S')}_{\F}\le L\norm{S-S'}_{\F}\) for all \(S,S'\),
and \(\norm{\nabla\ell(S)}_{\F}^2 \ge 2\kappa\bigl(\ell(S)-\ell_\star\bigr)\) for all \(S\).
\end{assumption}

We take \(0<\kappa\le L\), as holds automatically for any nonconstant loss
satisfying Assumption~\ref{ass:loss}.
For a nonzero quadratic Hessian $\quadH\succeq0$, one can take
$L=\lambda_{\max}(\quadH)$ and $\kappa=\lambda_{\min}^{+}(\quadH)$.
The operator $\quadH$ may have a nontrivial kernel.


\section{Mean-field conservation laws}
\label{sec:laws}\label{sec:conservation-hyperbolic}

\paragraph{Gradient flow and covariance closure}
Although training evolves all neuron weights, its effect on the
predictor is governed by a fixed-size covariance matrix.
We use the width-normalized gradient flow
$(\dot U_t,\dot V_t)=-n\nabla_{(U,V)}\cL(U_t,V_t)$, namely
\begin{equation}\label{eq:mean-field-gradient-flow}
 \dot U_t=-E_t^\top V_t,\qquad
 \dot V_t=-E_tU_t,\qquad E_t:=\nabla\ell(V_tU_t^\top/n).
\end{equation}
The factor $n$ fixes a common time scale across widths; equivalently
this is Euclidean gradient flow of $n\cL$.
With $W_t:=\binom{U_t}{V_t}$, define the uncentered covariance
\[
 \Sigma_t:=\tfrac1nW_tW_t^\top
 =\begin{psmallmatrix}\Suu_t&\Suv_t\\\Svu_t&\Svv_t\end{psmallmatrix},
 \qquad
 \Suu_t:=\tfrac{U_tU_t^\top}{n},\quad
 \Svv_t:=\tfrac{V_tV_t^\top}{n},\quad
 \Svu_t:=\tfrac{V_tU_t^\top}{n}.
\]
The hidden blocks $\Suu_t,\Svv_t$ do not enter the loss, but govern
the predictor's motion. We use the same notation $\Sigma$ whether the
covariance is specified directly or realized by finite-width factors:
the closed equation depends only on this matrix.
Set
$G_E:=\left(\begin{smallmatrix}0&E^\top\\E&0\end{smallmatrix}\right)$,
$G_t:=G_{E_t}$, and $G(S):=G_{\nabla\ell(S)}$, so $\dot W_t=-G_tW_t$.
The split structure of the gradient is encoded by the signature matrix
\[
 J:=\diag(\Id_{d_u},-\Id_{d_v}),\qquad G_tJ+JG_t=0.
\]
Our first result combines covariance closure with width-independent spectral conservation laws.
\begin{proposition}[Covariance closure and conservation laws]
\label{prop:covariance-closure-laws}
Let $\mathcal H_k(\Sigma):=\tr\!\left((J\Sigma)^k\right)$.
On its interval of existence, the factor flow
\eqref{eq:mean-field-gradient-flow} closes on the covariance $\Sigma_t$ as a Bures equation, for which the characteristic polynomial of $J\Sigma$ is preserved:
\begin{equation}\label{eq:bures-G}
 \dot\Sigma_t=-(G_t\Sigma_t+\Sigma_tG_t), \qquad
 \tfrac{\dd}{\dd t}\mathcal H_k(\Sigma_t)=0, \quad k=1,\ldots,m,
\end{equation}
\end{proposition}

The covariance equation \eqref{eq:bures-G} is the Bures gradient
flow~\citep{Takatsu2011,BhatiaJainLim2019} of
$F(\Sigma):=\ell(\Svu)$, since $\nabla_\Sigma F=G/2$.
Its fundamental matrix $\Phi_t$, defined by $\dot\Phi_t=-G_t\Phi_t$ and $\Phi_0=\Id$, satisfies
\begin{equation}\label{eq:orbit}
 \Phi_t^\top J\Phi_t=  \Phi_tJ\Phi_t^\top=J,\qquad
 \Sigma_t=\Phi_t\Sigma_0\Phi_t^\top.
\end{equation}
The functions $\mathcal H_k$ are conserved at every width $n$, defining mean-field laws.
Appendix~\ref{app:mean-field} gives their measure formulation, the corresponding Wasserstein PDE, and describes a larger family of laws.

\paragraph{Conserved spectral margins}
The spectral conservation laws \eqref{eq:bures-G} give uniform lower bounds on the spectrum of the hidden blocks $(\Suu, \Svv)$ through two margins
of the full covariance.
\begin{definition}[Covariance support margins and gap]\label{def:support-gap}
For $\Sigma\succeq0$, the support margins and their sum are
\setlength{\abovedisplayskip}{2pt}
\setlength{\belowdisplayskip}{2pt}
\setlength{\abovedisplayshortskip}{0pt}
\setlength{\belowdisplayshortskip}{2pt}
\begin{equation}\label{eq:margins}
 a_\pm(\Sigma):=\max\{a\ge0:\Sigma\mp aJ\succeq0\},\qquad
 \gap{\Sigma}:=\aplus{\Sigma}+\aminus{\Sigma}.
\end{equation}
\end{definition}
These maxima are finite and nonnegative. If $\lambda_1\le\cdots\le\lambda_m$
are the eigenvalues of $\Sigma^{1/2}J\Sigma^{1/2}$, \emph{including zeros}, then
\begin{equation}\label{eq:gap-spectrum}
 \aplus{\Sigma}=\lambda_{d_v+1},\qquad
 \aminus{\Sigma}=-\lambda_{d_v},\qquad
 \gap{\Sigma}=\lambda_{d_v+1}-\lambda_{d_v}.
\end{equation}
Thus $\gap{\Sigma}$ is an indexed spectral gap. Equivalently, it is the
length of the positive-semidefinite interval
$[-\aminus{\Sigma},\aplus{\Sigma}]$ of the pencil $\Sigma-tJ$;
such intervals are standard in matrix-pencil theory
\citep{NguyenNguyen2023}. Appendix~\ref{app:technical-details} proves
\eqref{eq:gap-spectrum}, including at singular covariances.
For $\Sigma\succ0$, $\aplus{\Sigma}$ and $-\aminus{\Sigma}$ are the
signed eigenvalues of $\Sigma^{1/2} J \Sigma^{1/2}$ nearest zero on their respective sides, and
$\gap{\Sigma}\ge2\lambda_{\min}(\Sigma)>0$.

\begin{proposition}[Transported lower bounds]\label{lem:margins}
On its interval of existence, every solution of \eqref{eq:bures-G} from
$\Sigma_0\succeq0$ preserves $a_\pm(\Sigma_t)=a_\pm(\Sigma_0)$ and satisfies
\[
  \Sigma_t \mp a_{\pm}(\Sigma_0) J \succeq 0, \qquad
  \Suu_t\succeq\aplus{\Sigma_0}\Id_{d_u},\qquad
  \Svv_t\succeq\aminus{\Sigma_0}\Id_{d_v}.
\]
\end{proposition}

The margins $a_\pm(\Sigma)$ need not equal the minimum eigenvalues of the initial hidden blocks: they satisfy
$\aplus{\Sigma}\le\lambda_{\min}(\Suu)$ and
$\aminus{\Sigma}\le\lambda_{\min}(\Svv)$, with equality when $\Svu=0$.
For instance, letting $\Sigma=\left(\begin{smallmatrix}1&r\\r&1\end{smallmatrix}\right)$,
$|r|\le1$, both margins are $\sqrt{1-r^2}$, whereas both hidden variances are one.
The support gap $g(\Sigma)$ controls predictor dissipation and can remain positive even when the full covariance is degenerate.
\begin{wrapfigure}[22]{r}{0.32\textwidth}
  \centering
  \resizebox{\linewidth}{!}{\input{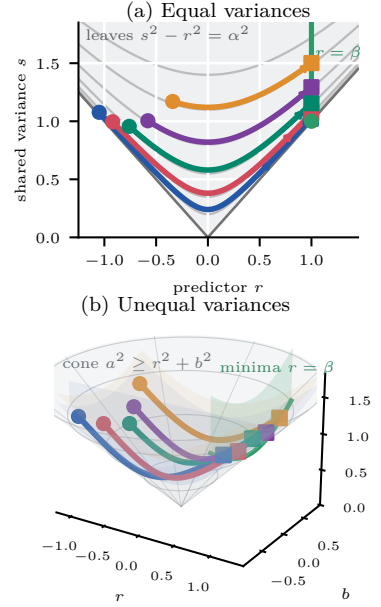}}

  \caption{Bures trajectories in the covariance cone.
  (a) Equal diagonal blocks: invariant hyperbolas.
  (b) Unequal blocks: fixed-\(b\) sections of determinant leaves.
  The displayed trajectories converge to \(r=\beta\).}
  \label{fig:gaussian-w2-cone-geometry}
\end{wrapfigure}
Under special conditions on the conserved spectrum of $J\Sigma$, elimination of hidden blocks becomes possible: Appendix~\ref{app:intrinsic} describes these intrinsic dynamics.

\paragraph{Numerical illustration: the $2\times2$ covariance cone}
For scalar input and output, every covariance has the form
$\Sigma=\begin{psmallmatrix}p&r\\r&q\end{psmallmatrix}\succeq0$.
Take the loss $F(\Sigma):=(r-\beta)^2/2$, with scalar teacher $\beta$.
Writing $(a,b):=((p+q)/2,(p-q)/2)$ identifies $\Sigma\succeq0$ with
the Lorentz cone $a^2\ge r^2+b^2$, $a\ge0$. The Bures flow
\eqref{eq:block-flow-general} becomes
$(\dot a,\dot b,\dot r)=-(2r(r-\beta),0,2a(r-\beta))$,
preserving $b=b_0$ and $\alpha^2:=a^2-r^2-b^2=pq-r^2$.

When $p=q$, set $s:=p$. Then $b=0$ and trajectories follow hyperbolas
$s^2-r^2=\alpha^2$ in $s\ge|r|$ (Figure~\ref{fig:gaussian-w2-cone-geometry}(a)).
For $\alpha>0$, they converge to $(r,s)=(\beta,\sqrt{\beta^2+\alpha^2})$;
the spectral-margin bound is exact for the minimum $s$ on the leaf.
Unequal diagonal blocks give intersections of the planes $b=b_0$ with
determinant leaves in the three-dimensional cone (panel (b)).
If $b_0^2+\alpha^2>0$, the endpoint is
$(r_\infty,a_\infty)=(\beta,\sqrt{\beta^2+b_0^2+\alpha^2})$.

The invariants $2b$ and $-\alpha^2$ are the trace and determinant of
$J\Sigma$, whose eigenvalues are $b_0\pm\sqrt{b_0^2+\alpha^2}$.
Thus $\gap{\Sigma_0}=2\sqrt{b_0^2+\alpha^2}$: the convergence bound
covers both panels, including singular leaves with $b_0\ne0$.
Remark~\ref{rem:degenerate-boundary-leaves} in Appendix~\ref{app:gf}
treats the zero-support-gap boundary. Figure~\ref{fig:gaussian-w2-cone-geometry}
is obtained by integrating $\dot r=-2\sqrt{r^2+b_0^2+\alpha^2}(r-\beta)$.


\section{Linear convergence of training dynamics}
\label{sec:convergence}\label{sec:balanced-convergence}
\label{sec:sampling}\label{sec:finite-width}

\paragraph{Convergence of gradient flow}
A positive conserved support gap $g(\Sigma_0) > 0$ guarantees global linear convergence for the Bures flow \eqref{eq:bures-G}, including some cases where the initial covariance $\Sigma_0$ is singular. Write $\Delta_t:=\ell(\Svu_t)-\ell_\star$ for the loss gap and $\cS_\star$ for the set of global predictor minimizers.

\begin{theorem}[Global convergence from a positive support gap]\label{thm:gf}\label{thm:balanced-convergence}
Under Assumption~\ref{ass:loss}, let $\Sigma_0\succeq0$ be a possibly
singular covariance with $\gap{\Sigma_0}>0$. 
Then the solution of \eqref{eq:bures-G} is global and satisfies
\begin{equation}\label{eq:gf-rate}
 \Delta_t\le\Delta_0e^{-2\kappa \gap{\Sigma_0}t},\qquad t\ge0.
\end{equation}
As $t \to \infty$, $\Sigma_t \to \Sigma_\infty\succeq0$ with
$\rankop\Sigma_\infty=\rankop\Sigma_0$ and $\Svu_\infty\in\cS_\star$.
The predictor $\Svu$ and any factor realization (as per \eqref{eq:mean-field-gradient-flow}) converge linearly with rate $\kappa\gap{\Sigma_0}$ in Frobenius norm.
\end{theorem}

The only nondegeneracy requirement is $\gap{\Sigma_0}>0$.
Positive definiteness is sufficient but not necessary: $\Sigma_0\succ0$ implies $\gap{\Sigma_0}\ge2\lambda_{\min}(\Sigma_0)>0$ and then
$\Sigma_\infty\succ0$. The exponent separates loss geometry (captured by the PL constant $\kappa$)
from initial covariance conditioning (the support gap $\gap{\Sigma_0}$), with no extra
dimension factor. The trajectory selects a global minimizer according to its initialization and dynamics. Proposition~\ref{prop:gf-lift-tail} supplies all explicit prefactors and bounds.

\begin{remark}[Singular covariance and predictor rank]\label{rem:singular-support}
The initialization $\Sigma_0=\diag(\Id_{d_u},0_{d_v})$ has
$\aplus{\Sigma_0}=1$, $\aminus{\Sigma_0}=0$, and $\gap{\Sigma_0}=1$.
Its rank is only $d_u$, yet the theorem applies. There is no rank
contradiction: a full-rank predictor needs rank $\min(d_u,d_v)$, not $m$.
More generally, $\aplus{\Sigma_0}>0$ implies $\rankop\Sigma_0\ge d_u$,
while $\aminus{\Sigma_0}>0$ implies $\rankop\Sigma_0\ge d_v$.
The condition cannot be dropped altogether: for $d_u=d_v=2$,
$\Sigma_0=\diag(1,0,1,0)$ and
$\ell(S):=\norm{S-e_2e_2^\top}_{\F}^2/2$, the flow is stationary and
nonoptimal. Here $g(\Sigma_0) = 0$, despite the nonzero signed eigenvalues $+1,-1$.
\end{remark}

\begin{remark}[Quadratic specialization]
For $\ell(S):=\frac12\ip{\quadH[S-\Svu_\star]}{S-\Svu_\star}$ with
$\quadH\succeq0$ a nonzero self-adjoint linear operator, the exponent is
$2\gap{\Sigma_0}\lambda_{\min}^{+}(\quadH)$. For a full-rank isotropic initialization
$\Sigma_0=\sigma^2\Id_m$, one has $\aplus{\Sigma_0}=\aminus{\Sigma_0}=\sigma^2$,
giving a rate of $4\sigma^2\lambda_{\min}^{+}(\quadH)$.
The operator $\quadH$ is allowed to have a nontrivial kernel: the predictor $\Svu$ converges to a point in
$\Svu_\star+\kerop\quadH$. If $\quadH$ is positive definite, the predictor converges to the teacher $\Svu_\star$.
\end{remark}

For a quadratic predictor loss, Figure~\ref{fig:bures-convergence-panel} compares the initialization scale, loss of conditioning, and the finite-width sampling we discuss next. The dotted curves use a coarse determinant-based rate, derived for comparison in Appendix~\ref{app:gf}.
If the support gap $g(\Sigma_0)$ vanishes, trace and initial loss alone cannot guarantee a uniform positive rate (Example~\ref{ex:slow}).
Figure~\ref{fig:smooth-lasso-panel} in Appendix~\ref{app:smooth-lasso} provides the counterpart for a non-quadratic smoothed-Lasso predictor loss.

\begin{figure}[t]
    \centering
    \resizebox{\linewidth}{!}{\input{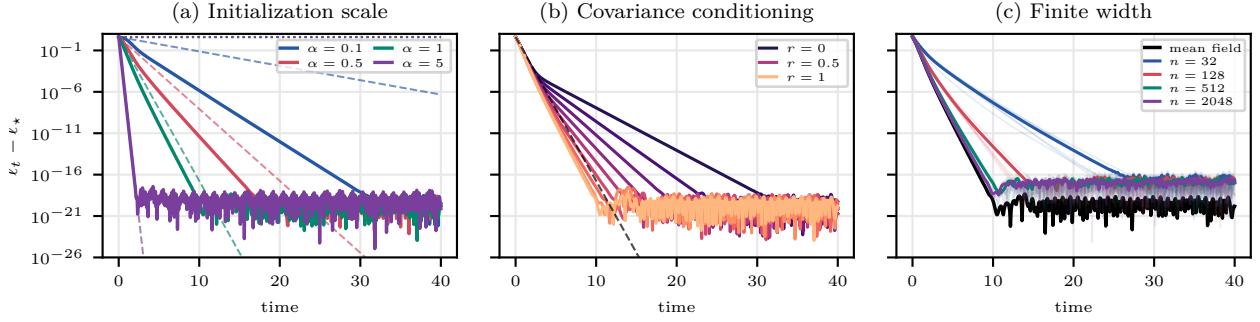}}

    \caption{
      Numerical convergence for quadratic loss with $d_u=16$, $d_v=9$, a random teacher, and whitened input.
      \textbf{(a)} Bures flow from $\Sigma_0=\alpha\Id_m$ at several white scales. Dashed lines show the spectral-margin rate $4\alpha\kappa$ from Theorem~\ref{thm:balanced-convergence}; dotted lines show the coarse rate $4\kappa c_{\rm coarse}$ from the determinant comparison in Appendix~\ref{app:gf} (Remark~\ref{rem:coarse-rate}), using the exact initial sublevel radius $M=2\norm{\Svu_\star}_{\F}$.
      \textbf{(b)} Rate deterioration when one matched diagonal pair of $\Sigma_0$ is scaled by $r$: $r=1$ is the white reference and $r=0$ is rank-deficient.
      \textbf{(c)} Bures flow from the reference covariance $\tSigma_0=\Id_m$ compared with finite-width flows initialized by i.i.d. samples from $\mathcal N(0,\tSigma_0)$, for $n=32,128,512,2048$ and five independent trials.
    }
    \label{fig:bures-convergence-panel}
    \label{fig:bures-scale-sweep}
    \label{fig:perturb-scale-sweep}
    \label{fig:bures-finite-width-comparison}
\end{figure}

\paragraph{Convergence under finite-width sampling}
Theorem~\ref{thm:gf} applies generally to neuron law covariances with positive support gap.
We now focus on finite-width networks, whose initial empirical covariance is $\Sigma_0:=W_0W_0^\top/n$.
For random initialization, the goal is to make this
empirical covariance close to a prescribed population covariance
$\tSigma_0$, thereby retaining its support gap and convergence rate.
Draw the columns $w_i:=(u_i,v_i)$ of $W_0:=\binom{U_0}{V_0}$ independently
from the same \emph{subgaussian} distribution:
$w_i:=\tSigma_0^{1/2}\xi_i$, where $\mathbb E[\xi_i\xi_i^\top]=\Id_m$ and $\sup_{\norm z=1}\norm{\langle z,\xi_i\rangle}_{\psi_2}\le K$.
Here $\norm X_{\psi_2}:=\inf\{c>0:\mathbb E e^{X^2/c^2}\le2\}$ denotes the standard subgaussian norm \citep{Vershynin2018};
standard Gaussian sampling has dimension-independent $K$.
Thus $\mathbb E[w_iw_i^\top]=\tSigma_0$, whereas $\Sigma_0=W_0W_0^\top/n$ is the empirical covariance of finitely many sampled neurons.
The loss gap along the sampled trajectory remains $\Delta_t=\ell(\Svu_t)-\ell_\star$ and we write $\widetilde\Delta_0:=\ell(\tSvu_0)-\ell_\star$ for its population reference.

\begin{corollary}[Finite width preserves the population rate]
\label{cor:finite-width}\label{thm:finite-width}
Under Assumption~\ref{ass:loss}, let $\tSigma_0\succeq0$ with
$\gap{\tSigma_0}>0$, set $r_0:=\rankop\tSigma_0$, and fix
$\varepsilon,\delta\in(0,1)$.
There is a universal $C>0$ such that, if
\begin{equation}\label{eq:width}
 n\ge CK^4\varepsilon^{-2}\bigl(r_0+\log(2/\delta)\bigr),
\end{equation}
then with probability at least $1-\delta$, it holds that 
\((1-\varepsilon)\tSigma_0\preceq\Sigma_0 \preceq(1+\varepsilon)\tSigma_0\).
On this event,
$\gap{\Sigma_0}\ge(1-\varepsilon)\gap{\tSigma_0}$,
the empirical factor flow \eqref{eq:mean-field-gradient-flow} is global, and
\(\Delta_t\le\Delta_0e^{-2\kappa(1-\varepsilon)\gap{\tSigma_0}t}\)
for all $t\ge0$.
The factors $(U_t, V_t)$ converge to a limit with globally optimal predictor $\frac{1}{n} V_\infty U_\infty^\top = \Svu_\infty$.
\end{corollary}

\begin{remark}[Width and initialization geometry]\label{rem:rank-adaptive-width}
The condition $\gap{\tSigma_0}>0$ forces $r_0\ge\min(d_u,d_v)$,
and this minimum is attainable. If $d_u\le d_v$, take
$\tSigma_0=\diag(2\sigma^2\Id_{d_u},0_{d_v})$, $\sigma>0$:
then $r_0=d_u$ and $\gap{\tSigma_0}=2\sigma^2$.
With Gaussian sampling, Corollary~\ref{cor:finite-width} gives width proportional to
$\varepsilon^{-2}(d_u+\log(2/\delta))$ and rate $4(1-\varepsilon)\kappa\sigma^2$, the same as for full-rank
population covariance $\sigma^2\Id_m$. Here every $v_i$ starts at zero,
so $\Svu_0=0$ and $\Delta_0=\widetilde\Delta_0$ exactly.
Exchanging the blocks gives dependence on $\min(d_u,d_v)$ in general.
This improvement changes the initialization: for $\tSigma_0=\sigma^2\Id_m$,
$r_0=m$. Its empirical support gap can vanish at small width even though its
population support gap is positive (Example~\ref{ex:white-small-width}).
\end{remark}

For the sampling scheme in Corollary~\ref{cor:finite-width}, the prefactor is
the empirical initial loss gap $\Delta_0$, which need not equal the population
loss gap $\widetilde\Delta_0$. Controlling it is a separate
accuracy requirement: for isotropic Gaussian sampling and the standard quadratic loss,
the expected excess initial error is $\sigma^4d_ud_v/(2n)$
(Proposition~\ref{prop:gaussian}). Theorem~\ref{thm:sampling-budgets}
gives a separate accuracy budget and a deterministic loss envelope.
For population variance $\sigma^2 \Id_m$ and $\varepsilon=1/2$,
the certified exponent is $2\kappa\sigma^2$.


\paragraph{Convergence of gradient descent}
Unlike finite-width sampling of neurons, discretized gradient descent in the factor matrices $(U, V)$ with a finite step size changes the covariance evolution itself.
We show that sufficiently small step sizes retain global linear convergence by controlling the accumulated conservation error.
We consider discrete-time gradient descent (GD) on the factor matrices, using an Euler discretization of \eqref{eq:mean-field-gradient-flow} with step size $\tau$.
As in the gradient flow setting, define $\Sigma_k:=W_kW_k^\top/n$, $\Svu_k:=V_kU_k^\top/n$, and $\Delta_k:=\ell(\Svu_k)-\ell_\star$.
Setting $E_k:=\nabla\ell(\Svu_k)$ and $G_k:=G_{E_k}$, the iterates read
\begin{equation}\label{eq:gd-algorithm}
 U_{k+1}=U_k-\tau E_k^\top V_k,\qquad
 V_{k+1}=V_k-\tau E_kU_k.
\end{equation}
Equivalently, $W_{k+1}=W_k-n\tau\nabla_W\cL(U_k,V_k)$: this is GD for $n\cL$ with stepsize $\tau$, or for $\cL$ with stepsize $n\tau$.
Writing $W_{k+1}=M_kW_k$ with $M_k:=\Id-\tau G_k$ gives
\begin{equation}\label{eq:gd-covariance}
 \Sigma_{k+1}=M_k\Sigma_kM_k
 =\Sigma_k-\tau(G_k\Sigma_k+\Sigma_kG_k)+\tau^2G_k\Sigma_kG_k.
\end{equation}
The additional term $\tau^2G_k\Sigma_kG_k$ arises from taking the GD step in the individual factors $(U,V)$ and is absent from an Euler step taken directly on the Bures ODE.

A positive initial support gap $g(\Sigma_0)$ still guarantees global linear convergence for sufficiently small step sizes $\tau$. 
The condition on $\tau$ is given below and depends only on the initial covariance, initial loss gap, and loss parameters; it controls both 
descent and the accumulated conservation error.

\begin{theorem}[Small-step global convergence of factor descent]
\label{thm:gd}\label{thm:discrete-bootstrap}
Under Assumption~\ref{ass:loss}, let $\Sigma_0:=W_0W_0^\top/n$ satisfy
$\gap{\Sigma_0}>0$, and suppose $\Delta_0>0$. Set
$I_\star:=4\sqrt{\Delta_0/\kappa}/\gap{\Sigma_0}$. If
\begin{equation}\label{eq:gd-step}
 0<\tau\le\tau_{\rm GD}:=\gap{\Sigma_0}e^{-2I_\star}\min\left\{
 \dfrac{1}{4\norm{\Sigma_0}_{\op}\sqrt{L\Delta_0}
              +9\sqrt2 L\norm{\Sigma_0}_{\op}^2e^{2I_\star}},\quad
 \dfrac{\sqrt2-1}{4\Delta_0}
 \right\},
\end{equation}
then, with $q:=1-\kappa\tau\gap{\Sigma_0}/\sqrt2\in(0,1)$ it holds that \(\Delta_k\le q^k\Delta_0\).
The factors $W_k$ and covariance $\Sigma_k$ converge at rate $q^{k/2}$ to limits
$W_\infty$ and $\Sigma_\infty\succeq0$, with
$\rankop\Sigma_\infty=\rankop\Sigma_0$ and
$\Svu_\infty\in\cS_\star$; explicit prefactors are given in
\eqref{eq:gd-tails}.
\end{theorem}

The support-gap condition $g(\Sigma_0) > 0$ is implied by $\Sigma_0\succ0$, in which case $\Sigma_\infty$ is also positive definite.
The threshold $\tau_{\rm GD}$ is positive for every initialization with $\gap{\Sigma_0}>0$ and $\Delta_0 < \infty$; no initial loss gap assumption or dimension-dependent term is needed.
The first restriction ensures descent, and the second controls the accumulated conservation error.

Appendix~\ref{app:gd} gives the explicit expression for $\tau_{\rm GD}$ when
$\Sigma_0=\sigma^2\Id_m$ and a comparison with the step-size condition obtained
by linearizing near an equilibrium (Remark~\ref{rem:gd-local}).
Proposition~\ref{prop:gd-adjustable} shows that sufficiently small step sizes
yield a certified loss-decay exponent arbitrarily close to $2\kappa\gap{\Sigma_0}$, the gradient-flow exponent in
Theorem~\ref{thm:balanced-convergence}.
Corollary~\ref{cor:sampled-gd} gives a step size chosen before sampling the
factors, using the loss parameters and the bounds on the initial support gap,
covariance norm, and loss gap from Theorem~\ref{thm:sampling-budgets}.


\section{Deep linear residual networks}
\label{sec:resnets}\label{sec:deep-resnets}
We now turn to deep linear residual networks (ResNets), including a continuous-depth formulation.
To stabilize optimization as depth increases, the networks considered here use skip connections between successive layers and scale each residual update by the inverse depth.
In this section, we demonstrate that linear ResNets also enjoy global exponential convergence rates that remain uniform across depths,
provided the average conserved layer support gap is bounded away from zero and the residual predictors satisfy an a priori operator-norm bound independent of depth throughout training.
The latter is an additional hypothesis: layerwise conservation alone does not control forward and backward propagation.
For simplicity, we state the results for gradient flow, as in Section~\ref{sec:balanced-convergence}.

\paragraph{Finite-depth setting}
We consider a depth-$S$, width-$n$ network with matching input and output dimensions $d_u=d_v=d$, and set $h:=1/S$.
For $s=0,\ldots,S-1$, layer $s$ has factors $U_s,V_s\in\R^{d\times n}$, covariance $\Sigma_s:=W_sW_s^\top/n$ with $W_s:=\binom{U_s}{V_s}$, and predictor $\Svu_s:=V_sU_s^\top/n$.
The network maps an input $x^0\in\R^d$ to its output $x^S$ through
\[
 x^{s+1}=x^s+\tfrac{1}{Sn}V_sU_s^\top x^s=A_sx^s, \qquad A_s:=\Id+h\Svu_s,\qquad s=0,\ldots,S-1.
\]
Thus $x^S=P_Sx^0$, where the endpoint predictor is the ordered product
\[
 P_S:=A_{S-1}\cdots A_0
 =\left(\Id+\tfrac{V_{S-1}U_{S-1}^\top}{Sn}\right)\cdots
  \left(\Id+\tfrac{V_0U_0^\top}{Sn}\right).
\]
Backpropagation in the factors couples the layers through these products, which can be singular even when every layer covariance is positive definite.
Let $P_{>s}:=A_{S-1}\cdots A_{s+1}$ and $P_{<s}:=A_{s-1}\cdots A_0$, with empty products equal to $\Id$.
For $E:=\nabla\ell(P_S)$, the backpropagated gradient is $\widetilde E_s:=P_{>s}^\top E P_{<s}^\top$.
In residual mean-field time, each layer obeys
\begin{equation}\label{eq:resnet-layer-bures}
 \dot\Sigma_s=-(\widetilde G_s\Sigma_s+\Sigma_s\widetilde G_s),
 \qquad \widetilde G_s:=G_{\widetilde E_s}.
\end{equation}
This is the covariance equation corresponding to factor gradient flow for $nS\ell(P_S)$. Since $\widetilde G_sJ+J\widetilde G_s=0$,
the signed margins are conserved separately at every layer.

\Needspace{14\baselineskip}
\begin{theorem}[Finite-depth convergence from conserved layer margins]\label{thm:resnet-finite}
Under Assumption~\ref{ass:loss} for the endpoint loss, suppose
$\Sigma_s(0)\succ0$ for all $s$, and write
$\overline g_S:=S^{-1}\sum_{s=0}^{S-1}\gap{\Sigma_s(0)}$ for the average initial
support gap. Suppose
$\sup_s\norm{\Svu_s(t)}_{\op}\le B<S$ for some constant $0 \leq B < \infty$ on the maximal existence interval,
and define the propagation lower bound $\underline{\sigma}_S:=(1-B/S)^{S-1}$.
Then the flow is global and
\(\Delta_t:=\ell(P_S(t))-\ell_\star
 \le\Delta_0e^{-2\kappa\overline g_S\underline{\sigma}_S^2t}\).
Every layer's covariance and factors converge, and the limiting endpoint
is globally optimal.
\end{theorem}

The rate combines the average support gap $\overline g_S$ with
$\underline{\sigma}_S^2$. For fixed $B$, this product-conditioning factor tends to
$e^{-2B}$ and is at least $e^{-4B}$ when $S\ge2B$.
For fixed $B$, a positive lower bound on the average layer support gap suffices for a convergence rate uniform in depth.

\paragraph{Continuous-depth limit}
Formally in the limit $S \to \infty$, residual products become propagators of a linear ODE.
Theorem~\ref{thm:resnet-finite} can be extended to a continuum of residual layers when these propagators stay uniformly conditioned, but requires a well-defined global training trajectory.
Let $s\in[0,1]$ denote depth. At time $t$, the network maps $x$ to
$P(t)x$, where $P(t):=X_1(t)$ and $X_s(t)$ solves
\[
 \partial_s X_s(t)=\Svu_s(t)X_s(t),\qquad X_0(t)=\Id.
\]
Write $P_{b,a}(t)$ for the propagator from depth $a$ to depth $b$. The backpropagated gradient at depth $s$ is
$\widetilde E_s:=P_{1,s}^\top\nabla\ell(P)P_{s,0}^\top$.
We assume the covariance profile $s\mapsto\Sigma_s(t)$ is defined for all
training times $t\ge0$ and depends locally absolutely continuously on $t$
as an element of $L^1([0,1];\Ssym^{2d})$.
The layer covariances are positive semidefinite and satisfy
\eqref{eq:resnet-layer-bures} in $L^1$. Its precise time-integral form is
\eqref{eq:resnet-integral} in Appendix~\ref{app:resnets}.

\begin{theorem}[Continuous-depth convergence from averaged margins]
\label{thm:resnet-continuous}
Under Assumption~\ref{ass:loss} for the endpoint loss, consider a covariance profile satisfying the assumptions above, with $\Sigma_s(0)\succ0$ for almost every $s$.
Set $\overline g:=\int_0^1\gap{\Sigma_s(0)}\,\dd s\in(0,\infty)$.
If $\sup_{t\ge0}\int_0^1\norm{\Svu_s(t)}_{\op}\,\dd s\le B<\infty$, then
\begin{equation}\label{eq:continuous-resnet-rate}
 \ell(P(t))-\ell_\star
 \le(\ell(P(0))-\ell_\star)e^{-2\kappa\overline g e^{-2B}t}.
\end{equation}
As $t\to\infty$, $\Sigma_s(t) \to \Sigma_s^\infty$ for almost every
depth $s$, and profiles converge in $L^1([0,1];\Ssym^{2d})$.
The network defined by the limiting profile attains loss $\ell_\star$.
Moreover, $\Sigma_s^\infty\succ0$ for almost every $s$, although its smallest
eigenvalue need not be bounded away from zero uniformly over depth.
\end{theorem}

The need for additional residual-path boundedness assumptions arises because the loss depends only on the endpoint $s = 1$: see Remark~\ref{rem:resnet-path-escape} in Appendix~\ref{app:resnets}.


\section{Heavy-ball momentum convergence}
\label{sec:momentum}
We now turn our attention to the analogous dynamics in an inertial setting that models heavy-ball momentum-based optimization.
Upon introducing momentum, the system no longer closes on the position covariance of the neuron weights \((u,v)\) alone: one must instead consider the phase space of positions and velocities tracked jointly.
We show that, in terms of the \emph{phase covariance}, dynamics are again closed for linear networks.
We derive conservation identities describing this phase geometry; energy dissipation and control of position deformation give the convergence certificate below.

\paragraph{Dynamics and phase closure}
Inertia adds velocity as a state variable, so the natural moment over which the dynamics close is the joint position-velocity covariance.
For $W:=\binom UV$, predictor $\Svu:=VU^\top/n$, and damping $\gamma>0$,
write $E:=\nabla\ell(\Svu)$. The factor dynamics are
\begin{equation}\label{eq:finite-momentum}
 \ddot W+\gamma\dot W+G(\Svu)W=0,\qquad
 G(\Svu):=\begin{psmallmatrix}0&E^\top\\E&0\end{psmallmatrix}.
\end{equation}
Write $(P,Q):=(\dot U,\dot V)$, so the columns of $W$ and $\dot W$ are
positions $x_i:=(u_i,v_i)$ and velocities $y_i:=(p_i,q_i)$ respectively.
In this section $\Phase:=n^{-1}\binom W{\dot W}\binom W{\dot W}^{\!\top}$
denotes the full $2m$-dimensional phase covariance. Its blocks are
$\Spos:=WW^\top/n$, $\Svel:=\dot W\dot W^\top/n$, and
$\Xi:=W\dot W^\top/n$; the predictor remains the $vu$ block of $\Spos$.
With $G_t:=G(\Svu_t)$, the closed equation is
\begin{equation}\label{eq:phase}
 \dot\Phase_t=\Mob_t\Phase_t+\Phase_t\Mob_t^\top,\qquad
 \Phase_t=\begin{psmallmatrix}\Spos_t&\Xi_t\\\Xi_t^\top&\Svel_t\end{psmallmatrix},\qquad
 \Mob_t:=\begin{psmallmatrix}0&\Id_m\\-G_t&-\gamma\Id_m\end{psmallmatrix}.
\end{equation}
Appendix~\ref{app:momentum-proof} gives the particle equations, measure
formulation, and blockwise expansion.

\paragraph{Conformal laws and mechanical energy}
Without damping, the fundamental matrix generated by $\Mob_t$ in \eqref{eq:phase} is symplectic;
damping rescales the associated bilinear form exponentially in time. 
\begin{proposition}[Conformal phase conservation laws]\label{prop:phase-laws}
Let $\Jcan:=\begin{psmallmatrix}0&\Id_m\\-\Id_m&0\end{psmallmatrix}$.
The phase fundamental matrix \(\dot\Psi_t=\Mob_t\Psi_t\),
\(\Psi_0=\Id_{2m}\) generated by $\Mob_t$ satisfies
\[
 (\Psi_t)^\top \Jcan\Psi_t=e^{-\gamma t}\Jcan,\qquad
 \Phase_t=\Psi_t\Phase_0(\Psi_t)^\top.
\]
Consequently, the spectrum of \(e^{\gamma t}\Jcan\Phase_t\)
and all its trace powers are constant. At the characteristic level,
these identities imply that, for any two trajectories of
the same flow, \(z_a:=(x_a,y_a)\), \(z_b:=(x_b,y_b)\),
\(e^{\gamma t}z_a^\top \Jcan z_b\) is conserved.
\end{proposition}

Special spectral values permit reduced phase dynamics, we describe these in Appendix~\ref{app:intrinsic}.
As before, write $\Delta_t:=\ell(\Svu_t)-\ell_\star$ for the predictor loss gap. Define $\mathcal{K}_t:=\tr(\Svel_t)/2$ and $\cE_t:=\Delta_t+\mathcal{K}_t$ as the kinetic and mechanical excess energies. Differentiating $\cE_t$ along \eqref{eq:phase} gives $\dot{\cE}_t=-2\gamma\mathcal{K}_t$.
This identity gives global existence of the solution to \eqref{eq:phase} from every $\Sigma_0 \succeq 0$ (Proposition~\ref{prop:global-phase}).
The remaining task is to derive spectral lower bounds for $\Spos_t$.

\paragraph{Global convergence}
We seek an initialization condition on $\Sigma_0$ that keeps $\Spos_t$ uniformly elliptic for the whole trajectory.
Let $a_0:=\lambda_{\min}(\Spos_0)>0$ and $b_0:=\lambda_{\max}(\Spos_0)$.
A deformation budget $s>0$ specifies a spectral interval $[a_s,b_s]:=[a_0e^{-2s},b_0e^{2s}]$ for $\Spos_t$.
For $\cE_0>0$, put $\rho_s:=4b_sL+2\sqrt{2L\cE_0}$ and define
\begin{equation}\label{eq:certificate}
 \lambda_s(\gamma):=\tfrac{a_s\kappa\gamma}{\rho_s+2\gamma^2},\qquad
 \mathfrak S_s(\gamma):=\tfrac4{\lambda_s(\gamma)}
             \sqrt{\tfrac{3\cE_0}{2a_s}}.
\end{equation}
Here, $\lambda_s$ is the exponential decay rate of $\cE_t$, and $\mathfrak S_s$ bounds the resulting total deformation in $\Spos_t$, conditional on the spectral interval $[a_s, b_s]$.

\begin{theorem}[Global linear rate from a phase-energy certificate]\label{thm:momentum}\label{thm:bootstrap-unconditional-rate}
Under Assumption~\ref{ass:loss}, let $\Phase_0\succeq0$ with
$\Spos_0\succ0$ and $\cE_0>0$. If
$\mathfrak S_s(\gamma)<s$ for some $s>0$, then for all $t\ge0$,
\[
 a_s\Id_m\preceq\Spos_t\preceq b_s\Id_m,\qquad
 \cE_t\le3\cE_0 e^{-\lambda_s(\gamma)t}.
\]
The covariance $\Spos_t$ and any associated factors
$(U_t,V_t)$ evolving under \eqref{eq:finite-momentum} converge exponentially with rate $\lambda_s/2$ to limits with optimal predictor and $\Spos_\infty\succ0$; $\Svel_t$ tends to zero.
\end{theorem}

\begin{corollary}[Two-scale white initialization]
\label{cor:two-scale-white-window}
Under Assumption~\ref{ass:loss}, let
\(\Spos_0=\sigma_{vu}^2\Id_m\), \(\Xi_0=0\), and
\(\Svel_0\succeq0\) have trace \(m\sigma_{qp}^2\), with
\(\sigma_{vu}>0\), \(\sigma_{qp}\ge0\).
Then \(\cE_0=\Delta+m\sigma_{qp}^2/2\), where
\(\Delta:=\ell(0)-\ell_\star\).
If \(\cE_0>0\), the certificate is
\[
 2\gamma^2-
 \tfrac{s\kappa\sigma_{vu}^3e^{-3s}}
 {4\sqrt{3\cE_0/2}}\gamma
 +4L\sigma_{vu}^2e^{2s}
 +2\sqrt{2L\cE_0}<0.
\]
For fixed \(s,\gamma>0,\Delta,\sigma_{qp},m,L,\kappa\), sufficiently
large \(\sigma_{vu}\) satisfies this condition.
\end{corollary}

Appendix~\ref{app:large-position-scale} gives explicit bounds on the admissible $\gamma$ interval and analyzes the large-$\sigma_{vu}$ asymptotics.

\paragraph{Numerical view of the bootstrap certificate}
\begin{wrapfigure}{r}{0.34\linewidth}
\centering
\resizebox{\linewidth}{!}{\subimport{figures/}{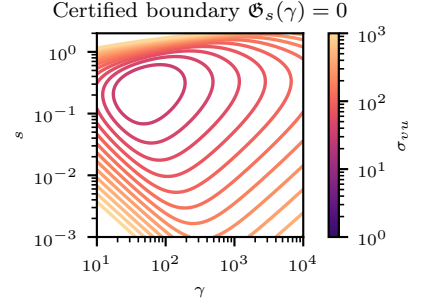}}
\footnotesize

\caption{Two-scale white momentum initialization. Curves show
$\mathfrak G_s(\gamma)=0$ for the color-bar scale $\sigma_{vu}$;
their interiors are certified damping windows.}
\label{fig:momentum-bootstrap-heatmap}\label{fig:momentum-window}
\end{wrapfigure}
Writing $\mathfrak G_s(\gamma):=\mathfrak S_s(\gamma)-s$, Figure~\ref{fig:momentum-bootstrap-heatmap} shows the boundaries
$\mathfrak G_s(\gamma)=0$ for $1\le\sigma_{vu}\le10^3$, with $(d_u,d_v)=(3,2)$,
$\kappa=L=1$, $\sigma_{qp}=1$, and $\ell(0)-\ell_\star=1$.
The regions where $\mathfrak G_s(\gamma)<0$ specify the pairs $(s,\gamma)$ for which Theorem~\ref{thm:momentum} guarantees spectral bounds and exponential energy decay.
For the two-scale initialization, sufficiently large $\sigma_{vu}$ yields a nonempty interval of $\gamma$ satisfying Corollary~\ref{cor:two-scale-white-window}.
For fixed initialization and damping, a smaller admissible $s$ gives a
tighter tube and faster rate; Figure~\ref{fig:momentum-bootstrap-lambda-multi-sigma}
illustrates the limiting rate at the feasibility boundary analyzed in
Appendix~\ref{app:large-position-scale}.

Figure~\ref{fig:momentum-empirical-vs-theory} compares empirical losses with the upper bounds on $\cE_t$ from Corollary~\ref{cor:explicit-tightened-conditional-momentum-rate}.
Each upper bound curve uses an explicit, approximately optimal rate certified by Theorem~\ref{thm:momentum} (see Proposition \ref{prop:optimal-tightened-conditional-momentum-rate}); no rate is fitted
to the empirical trajectory. All displayed losses stay below their upper bounds.
Notably in this inertial setting, the predictor loss may oscillate even while the mechanical energy decreases.

\begin{figure}[!htbp]
\centering
\resizebox{\linewidth}{!}{\subimport{figures/}{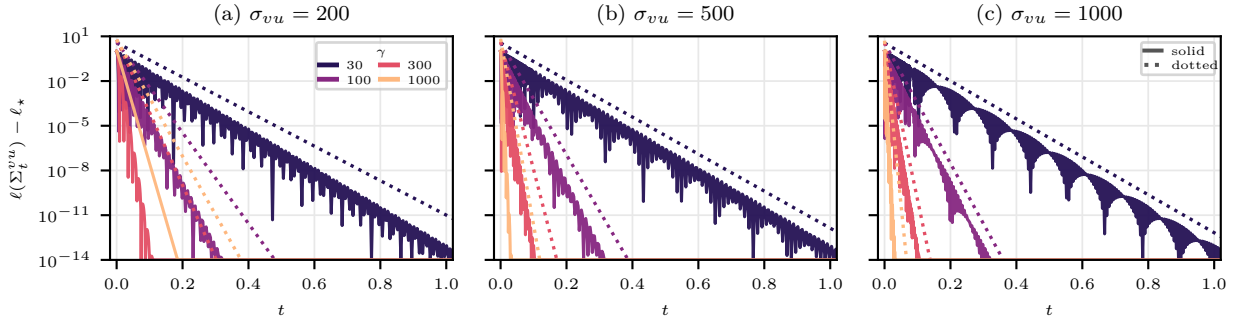}}

\caption{
Empirical predictor-loss trajectories (solid) and the corresponding upper bounds
on mechanical excess energy (dotted) for two-scale white initializations with varying
\(\sigma_{vu}\) and damping \(\gamma\): \textbf{(a)} $\sigma_{vu}=200$,
\textbf{(b)} $\sigma_{vu}=500$, and \textbf{(c)} $\sigma_{vu}=1000$.
Each theoretical curve uses the explicit, approximately optimal rate on a tube certified by
Theorem~\ref{thm:bootstrap-unconditional-rate}.
}
\label{fig:momentum-empirical-vs-theory}\label{fig:momentum-envelopes}
\end{figure}
\WFclear


\clearpage
\section{Conclusion}
The present work addresses two-layer linear networks, making use of the exact closure of mean-field training on the neuron law covariance as a finite-dimensional Bures flow.
This flow enjoys conservation laws that control the ellipticity of the predictor preconditioner through a conserved spectral support gap, yielding global linear convergence whenever this gap is positive.
In particular, this is guaranteed for every positive-definite initial covariance and for smooth PL predictor losses. 
We show that this guarantee survives finite-width sampling of neurons, and extends to discrete-time gradient descent in the factors through approximate versions of the same conservation laws.
The covariance viewpoint clarifies two further extensions: for deep linear ResNets, the same conservation law holds layerwise and gives linear convergence conditional on bounded residual paths.
Heavy-ball momentum dynamics close in the joint phase covariance, where conformal conservation laws and mechanical energy lead to an explicit damping certificate from general initial phase moments.
The certificate includes two-scale white initializations without imposing a phase-covariance identity. Covariance conditioning is the common quantitative mechanism throughout our analysis.


\subsection*{Acknowledgments}
Gabriel Peyr\'e acknowledges support from the European Research Council through
the WOLF project, and from France 2030 through grant ANR-23-IACL-0008
(PRAIRIE-PSAI), administered by the Agence Nationale de la Recherche.

\subsection*{AI use statement}
In this work, we used generative AI tools including GPT 5.6 and GPT 6 for aiding and polishing writing, research ideation and execution, drafting some sections of the paper, and deriving or proving mathematical claims as directed by the authors. Where outputs were produced by generative AI, these were refined through extensive iteration, verification, and manual editing of the results and exposition. 
We have reviewed all AI-assisted work, having verified the claims, proofs and other material presented in this paper.
We take responsibility for the final content of this work, including text, claims or artifacts produced with the aid of generative AI.

\newpage 
\bibliography{references}
\bibliographystyle{plainnat}
\appendix
\clearpage
\section{Notation table}\label{app:notation}
The table collects the symbols, their meanings, and their first definitions.
\begingroup
\small
\renewcommand{\arraystretch}{1.12}
\begin{longtable}{@{}>{\raggedright\arraybackslash}p{0.21\linewidth}
 >{\raggedright\arraybackslash}p{0.53\linewidth}
 >{\raggedright\arraybackslash}p{0.20\linewidth}@{}}
\toprule
Notation & Meaning & First reference\\
\midrule
\endfirsthead
\toprule
Notation & Meaning & First reference\\
\midrule
\endhead
\(d_u,d_v,n,m\)&Input/output dimensions, width, and \(m:=d_u+d_v\)&Section~\ref{sec:model}\\
\(U,V,W:=\binom UV,\cL\)&Factor matrices, stacked positions, and factor loss&Section~\ref{sec:model}\\
\(\Sigma,\Suu,\Svv,\Svu\)&Uncentered moment, hidden blocks, and predictor&Section~\ref{sec:laws}\\
\(\ell,\ell_\star,\cS_\star\)&Predictor loss, infimum, and minimizer set&Assumption~\ref{ass:loss}; Section~\ref{sec:convergence}\\
\(L,\kappa,\lambda\)&Smoothness, PL, and QG constants&Assumption~\ref{ass:loss}; Eq.~\eqref{eq:quadratic-growth}\\
\(\quadH,C_z,\Svu_\star\)&Quadratic Hessian, input second moment, teacher&Section~\ref{sec:model}\\
\(\mu,\widehat\mu_W\)&Neuron law and empirical law&Appendix~\ref{app:mean-field}\\
\(E,G_E\)&Predictor gradient and symmetric split generator&Section~\ref{sec:laws}\\
\(J,\Phi_t\)&Input/output signature and position fundamental matrix&Section~\ref{sec:laws}\\
\(\mathcal H,\mathcal H_{\varphi,r}\)&Mean-field law and Gram-test law&Def.~\ref{def:mean-field-law}; Eq.~\eqref{eq:gram-mean-field-law}\\
\(\mathcal H_k,\Gamma\)&Spectral laws and labeled Gram imbalance&Propositions~\ref{prop:covariance-closure-laws}, \ref{prop:mean-field-conservation-laws}\\
\(a_\pm(\Sigma),\gap{\Sigma}\)&Covariance support margins and gap&Def.~\ref{def:support-gap}\\
\(\Delta_t,I_0,B_\infty\)&Loss gap, force-length bound, covariance upper bound&Section~\ref{sec:convergence}; Proposition~\ref{prop:gf-lift-tail}\\
\(\varepsilon,\zeta,\beta\)&Rate tolerance, absolute accuracy, concentration tolerance&Corollary~\ref{cor:finite-width}; Eq.~\eqref{eq:sampling-budgets}\\
\(\tSigma_0,\widetilde\Delta_0,r_0\)&Population covariance, reference loss gap, and rank; empirical covariance is $\Sigma_0$&Corollary~\ref{cor:finite-width}\\
\(\underline g,\overline\Delta,\overline B_0\)&Empirical support-gap lower bound, loss-gap and covariance-norm upper bounds&Eqs.~\eqref{eq:sampling-budgets}, \eqref{eq:sampling-event}\\
\(\tau,I_\star\)&GD step and force-length bound&Eq.~\eqref{eq:gd-algorithm}; Theorem~\ref{thm:gd}\\
\(\alpha_\star,M_{\nabla},B_\star\)&Proof-local GD descent scale, initial gradient bound, and covariance bound&Eq.~\eqref{eq:gd-constants}\\
\(M_k,\mathcal J_k,\mathcal D_k\)&Factor update, transported signature, and its defect&Eq.~\eqref{eq:gd-algorithm}, \eqref{eq:gd-signature}\\
\(S,A_s,P_S,\widetilde E_s\)&Residual depth, layer map, endpoint, and backpropagated gradient&Section~\ref{sec:resnets}\\
\(\overline g_S,\overline g\)&Finite/continuous depth averages of $\gap{\Sigma_s(0)}$&Theorems~\ref{thm:resnet-finite}, \ref{thm:resnet-continuous}\\
\(s,B\)&Continuous depth and residual-path bound&Section~\ref{sec:resnets}\\
\(\underline{\sigma}_S,\overline{\sigma}_S\)&Lower and upper residual-propagation bounds&Section~\ref{sec:resnets}; Appendix~\ref{app:resnets}\\
\(x:=(u,v),y:=(p,q)\)&Phase position and velocity variables&Section~\ref{sec:momentum}\\
\(\Phase,\Spos,\Xi,\Svel\)&Full phase moment and its position, mixed, velocity blocks&Eq.~\eqref{eq:phase}\\
\(\gamma,\Jcan\)&Damping and canonical symplectic matrix&Eq.~\eqref{eq:finite-momentum}, Proposition~\ref{prop:phase-laws}\\
\(\Mob_t,\Psi_t\)&Phase generator and characteristic matrix&Eq.~\eqref{eq:phase}; Proposition~\ref{prop:phase-laws}\\
\(D,Q\)&Velocity regression and residual covariance (Schur complement)&Eq.~\eqref{eq:regression}\\
\(\mathcal{K},\cE\)&Kinetic and total mechanical excess energies; $\cE_t=\Delta_t+\mathcal{K}_t$&Section~\ref{sec:momentum}\\
\(s,a_s,b_s,\rho_s,\lambda_s,\mathfrak S_s\)&Deformation budget, position tube, force bound, rate, certificate; \(s\) here is not depth&Eq.~\eqref{eq:certificate}\\
\(\sigma_{vu},\sigma_{qp}\)&White-position scale and matched per-coordinate velocity scale&Corollary~\ref{cor:two-scale-white-window}\\
\(\varepsilon,\eta,q,\mathcal V\)&Momentum energy correction, Young parameter, equivalence constant, Lyapunov energy&Proposition~\ref{prop:tightened}\\
\(\cN,\mathfrak C_0\)&Mean squared particle force and tightened initial energy prefactor&Lemma~\ref{lem:conditional}, Corollary~\ref{cor:tightened-bootstrap}\\
\(\mathfrak o(J),\xi_A\)&Matrix Lie algebra and infinitesimal congruence field&Appendix~\ref{app:law-counting}\\
\bottomrule
\end{longtable}
\endgroup


\section{Measure formulation and Wasserstein conservation laws}
\label{app:mean-field}
The covariance equations in the main text have the same meaning at both finite and infinite widths. In this section, we introduce the measure flow that captures training dynamics, define a law shared by all widths, and lift the spectral invariants to functionals of a neuron distribution.

\subsection{Neuron distributions and Wasserstein flow}
We describe the configuration of any number of neurons by their law: for $w:=(u,v)\in\R^m$, let
$\mu\in\cP_2(\R^m)$ and write its uncentered second moment
\begin{equation}\label{eq:measure-covariance}
 \Sigma(\mu):=\int ww^\top\,\dd\mu(w).
\end{equation}
For $n$ neurons $W:=(w_1,\ldots,w_n)$, the neuron law is an empirical distribution 
$\widehat\mu_W:=n^{-1}\sum_{i=1}^n\delta_{w_i}$, for which 
$\Sigma(\widehat\mu_W)=WW^\top/n$. A general covariance
$\Sigma\succeq0$ is also realized by a centered Gaussian distribution.
The loss functional is $\cF(\mu):=\ell(\Svu(\mu))$, so
$\cF(\widehat\mu_W)=\cL(U,V)$.

For a functional with a sufficiently regular first variation, the
Wasserstein gradient is the spatial gradient
$\nabla_{W_2}\cF(\mu)(w):=\nabla_w(\delta\cF/\delta\mu)(\mu,w)$ and its negative is the velocity in the continuity equation
\begin{equation}\label{eq:wasserstein-pde}
 \partial_t\mu_t+\nabla_w\cdot(\mu_t b_t)=0,\qquad
 b_t:=-\nabla_{W_2}\cF(\mu_t).
\end{equation}
The PDE is understood distributionally, including for empirical
measures. In the present split model,
$\delta\cF/\delta\mu=\tfrac12 w^\top G_tw$ up to an additive constant.
Consequently, 
\begin{equation}\label{eq:split-wasserstein-pde}
 \partial_t\mu_t=\nabla_w\cdot(\mu_tG_tw),\qquad
 \mu_t=(\Phi_t)_\#\mu_0,
\end{equation}
where $\Phi_t$ is the fundamental matrix in \eqref{eq:orbit}.
For the empirical neuron law, the characteristics of \eqref{eq:wasserstein-pde} are exactly
$\dot u_i=-E_t^\top v_i$, $\dot v_i=-E_tu_i$.

Conversely, given a solution $\Sigma_t$ of \eqref{eq:bures-G}, transporting
any initial distribution with second moment $\Sigma_0$ by
$\dot\Phi_t=-G_t\Phi_t$ produces a Wasserstein gradient flow with second
moment $\Sigma_t$. Indeed, its second moment and $\Sigma_t$ solve the same
linear equation with the same initial data. Matching covariances alone
does not ensure that a measure curve follows this gradient flow.

\subsection{Covariance-dependent objectives}
Moment closure is not specific to the off-diagonal predictor:
the gradient flow of any sufficiently regular objective depending on second moments closes in terms of a Bures equation.
\begin{proposition}[Second-moment closure]
\label{prop:moment-closure}
Let $F$ be a differentiable scalar function on an open neighborhood of
$\Sp^m$ in the symmetric matrices, and define $\cF(\mu):=F(\Sigma(\mu))$.
Let $\mu_t$ be a sufficiently regular Wasserstein gradient flow of $\cF$, and set $\Sigma_t:=\Sigma(\mu_t)$ and $D_t:=\nabla_\Sigma F(\Sigma_t)$.
Then the characteristics satisfy \(\dot w_t=-2D_tw_t\), and the second moment evolves autonomously as
\(\dot\Sigma_t=-2(D_t\Sigma_t+\Sigma_tD_t)\).
Thus the Wasserstein flow of $\cF$ closes exactly on the Bures gradient flow of $F$ on second moments, with the normalization used throughout the paper.
\end{proposition}

\begin{proof}[Proof of Proposition~\ref{prop:moment-closure}]
The first variation of $\cF$ at $\mu_t$, up to an additive constant, is
$\frac{\delta\cF}{\delta\mu}(\mu_t,w)=\ip{D_t}{ww^\top}=w^\top D_tw$.
Taking its Euclidean gradient in $w$ gives $2D_tw$ and hence the claimed characteristic equation.
Finally, writing the flow map as $w_0 \mapsto w_t$ and differentiating $\Sigma_t=\int w_tw_t^\top\,\mathrm{d}\mu_0(w_0)$ gives
\[
\dot\Sigma_t=\int \left(\dot w_tw_t^\top+w_t\dot w_t^\top\right)\,\mathrm{d}\mu_0(w_0)=-2(D_t\Sigma_t+\Sigma_tD_t).\qedhere
\]
\end{proof}

\subsection{Width consistency and covariance-level laws}
We seek conservation laws defined on neuron distributions that remain
conserved when evaluated on finite networks of any width.

\begin{definition}[Mean-field conservation law]\label{def:mean-field-law}
A scalar functional \(\mathcal H(\mu)\) is a mean-field conservation law
for a family of flows if it is constant along every admissible trajectory
of every flow in that family, for every admissible initialization.
It is covariance-level if \(\mathcal H(\mu)=h(\Sigma(\mu))\).
\end{definition}
For differentiable functionals, the chain rule expresses conservation as
orthogonality:
\[
 \frac{\dd}{\dd t}\mathcal H(\mu_t)
 =-\int\nabla_{W_2}\mathcal H(\mu_t)\cdot
          \nabla_{W_2}\cF(\mu_t)\,\dd\mu_t=0, 
\]
whenever taking gradients and the chain rule are well-defined.
The laws below also have a direct characteristic proof.
A single mean-field law \( \mathcal H\) induces laws
\(h_n(W):=\mathcal H(\widehat\mu_W)\) for every width, in the
finite-dimensional sense of \citet{MarcotteGribonvalPeyre2023}.
Unrelated laws found separately at each width need not arise this way.
Conversely, conservation at every finite width passes to the measure flow
whenever the empirical flows converge in \(W_2\) and \(\mathcal H\) is
continuous in that metric. For covariance-level laws no limiting argument is needed:
every positive semidefinite covariance has an empirical realization, and
the common closed moment equation gives the equivalence with Bures laws.

\Needspace{10\baselineskip}
\subsection{Gram-function laws and their spectral restriction}
Preservation of pairwise indefinite products yields functionals
that need not depend only on covariance.
\begin{proposition}[Fundamental mean-field laws]
\label{prop:mean-field-conservation-laws}
Let $\mu_t=(\Phi_t)_\#\mu_0$ be the characteristic flow of
$\dot w_t=-G_tw_t$, where $G_t^\top=G_t$ and $G_tJ+JG_t=0$.
For every $r\ge1$ and Borel function $\varphi:\Ssym^r\to\R$,
the functional
\begin{equation}\label{eq:gram-mean-field-law}
 \mathcal H_{\varphi,r}(\mu)
 :=\int\varphi\!\left((w_a^\top Jw_b)_{1\le a,b\le r}\right)
       \prod_{a=1}^r\dd\mu(w_a)
\end{equation}
is conserved whenever this integral is absolutely integrable at
initialization. In particular, the spectral laws
$\mathcal H_k(\Sigma(\mu))$ of \eqref{eq:bures-G}
are conserved for every $\mu_0\in\cP_2(\R^m)$.
For particles, the Gram imbalance
$\Gamma_t:=V_t^\top V_t-U_t^\top U_t=-W_t^\top JW_t$ is constant and
\begin{equation}\label{eq:finite-spectrum-gamma}
 \mathcal H_k(\Sigma_t)=\tr\!\left((-\Gamma_t/n)^k\right),
 \qquad k\ge1.
\end{equation}
The nonzero spectra of $J\Sigma_t$ and $-\Gamma_t/n$ coincide.
\end{proposition}

\begin{proof}[Proof of Proposition~\ref{prop:mean-field-conservation-laws}]
Along characteristics,
\(
\frac{\mathrm{d}}{\mathrm{d}t}(w_a^\top Jw_b)
=-w_a^\top(G_tJ+JG_t)w_b=0,
\)
so \eqref{eq:gram-mean-field-law} is conserved.
Bounded Borel tests are integrable, and cyclic polynomial tests need
only second moments because each sampled variable appears quadratically.
For the cyclic polynomial test function, with \(w_{k+1}:=w_1\), repeated
integration gives
\[
\int\prod_{\ell=1}^k
\left(w_\ell^\top Jw_{\ell+1}\right)
\prod_{\ell=1}^k\mathrm{d}\mu(w_\ell)
=\tr\!\left((J\Sigma)^k\right).
\]
Thus \eqref{eq:bures-G} is a special case of
\eqref{eq:gram-mean-field-law}.
The first \(m\) trace powers determine the characteristic polynomial by Newton identities.
For empirical measures, $J\Sigma=(JW)(W^\top/n)$ while $(W^\top/n)(JW)=-\Gamma/n$.
The nonzero eigenvalues of $AB$ and $BA$ coincide, which gives the spectral statement and \eqref{eq:finite-spectrum-gamma}.
\end{proof}


\section{Technical details and proofs}\label{app:technical-details}\label{app:convergence-proofs}
\subsection{Loss estimates}\label{app:proofs}
\begin{lemma}[Smooth PL consequences]\label{lem:pl}
Under Assumption~\ref{ass:loss}, the minimizer set
$\cS_\star:=\{S:\ell(S)=\ell_\star\}$ is nonempty, and
\begin{equation}\label{eq:pl-consequences}
\norm{\nabla\ell(S)}_{\F}^2\le2L(\ell(S)-\ell_\star),\qquad
\dist(S,\cS_\star)^2\le2(\ell(S)-\ell_\star)/\kappa.
\end{equation}
For every perturbation $H$,
\begin{equation}\label{eq:perturb}
\sqrt{\ell(S+H)-\ell_\star}
\le\sqrt{\ell(S)-\ell_\star}+\sqrt{L/2}\norm H_{\F}.
\end{equation}
\end{lemma}

\begin{proof}[Proof of Lemma~\ref{lem:pl}]

Apply the descent lemma at $S-\nabla\ell(S)/L$ for the first inequality.
Since $\nabla\ell$ is globally Lipschitz, the Euclidean gradient flow
$\dot S=-\nabla\ell(S)$ exists for all $t\ge0$.
Writing $h:=\ell(S)-\ell_\star$ gives $\dot h=-\norm{\nabla\ell}^2$ and,
where $h>0$, $-\frac{\dd}{\dd t}\sqrt h\ge\sqrt{\kappa/2}\norm{\dot S}$.
Thus its total remaining length is at most $\sqrt{2h/\kappa}$.
The flow has a limit, with loss gap zero by PL, proving attainment and the distance
bound. Finally, the descent lemma and the first inequality bound the
perturbed loss gap by $(\sqrt h+\sqrt{L/2}\norm H_{\F})^2$.
\end{proof}
\begin{remark}[Quadratic growth and convexity]
Lemma~\ref{lem:pl} proves the quadratic-growth bound
\begin{equation}\label{eq:quadratic-growth}
 \ell(S)-\ell_\star\ge\frac\lambda2\dist(S,\cS_\star)^2
 \qquad\text{with }\lambda=\kappa.
\end{equation}
If a larger quadratic-growth constant $\lambda$ is known, it improves
the distance bounds while leaving the loss rates, which use $\kappa$, unchanged.
The PL-to-quadratic-growth implication does not need convexity
\citep{KarimiNutiniSchmidt2016}. Conversely, convex quadratic growth
with constant \(\lambda\) implies PL with constant \(\lambda/4\):
for a nearest minimizer \(S_\star\),
\(\Delta\le\ip{\nabla\ell(S)}{S-S_\star}
\le\norm{\nabla\ell(S)}_{\F}\sqrt{2\Delta/\lambda}\).
The best PL and quadratic-growth constants need not coincide.
For a nonzero quadratic Hessian $\quadH\succeq0$,
$\lambda=\kappa=\lambda_{\min}^{+}(\quadH)$ is valid and
$\cS_\star=\Svu_\star+\kerop\quadH$.
\end{remark}

\subsection{Covariance closure and conservation laws}
\label{app:covariance-laws}
Direct matrix differentiation proves closure and conservation at
every width. The measure interpretation is given separately in
Appendix~\ref{app:mean-field}.

\begin{proof}[Proof of Proposition~\ref{prop:covariance-closure-laws}]
Since $\dot W_t=-G_tW_t$ and $G_t^\top=G_t$,
differentiating $W_tW_t^\top/n$ gives \eqref{eq:bures-G}.
For any covariance solution, define
$\dot\Phi_t=-G_t\Phi_t$, $\Phi_0=\Id$. Uniqueness for the resulting
linear covariance equation yields
$\Sigma_t=\Phi_t\Sigma_0\Phi_t^\top$, so positive semidefiniteness
is preserved. The identity $G_tJ+JG_t=0$ gives
$\frac{\dd}{\dd t}(\Phi_t^\top J\Phi_t)=0$.
Invertibility of $\Phi_t$ then gives both identities in
\eqref{eq:orbit}. Equivalently,
$\frac{\dd}{\dd t}(J\Sigma_t)=[G_t,J\Sigma_t]$, so
$\frac{\dd}{\dd t}\tr((J\Sigma_t)^k)=0$ for every $k\ge1$.
The first $m$ trace powers determine the characteristic polynomial.
For later use, the block equations are
\begin{equation}\label{eq:block-flow-general}
 \begin{aligned}
 \dot\Suu_t&=-\Suv_tE_t-E_t^\top\Svu_t,\quad
 \dot\Svv_t=-\Svu_tE_t^\top-E_t\Suv_t,\quad
 \dot\Svu_t&=-E_t\Suu_t-\Svv_tE_t.
 \end{aligned}
\end{equation}
\end{proof}

\begin{proof}[Proof of Proposition~\ref{lem:margins}]
For any $c\in\R$, the characteristic congruence gives
\[
 \Sigma_t-cJ=\Phi_t(\Sigma_0-cJ)\Phi_t^\top.
\]
Invertibility of $\Phi_t$ preserves the entire feasible interval of $c$,
so both endpoints are conserved even when $\Sigma_0$ is singular.
Taking principal blocks proves the hidden bounds.
\end{proof}

\begin{proof}[Spectral formula \eqref{eq:gap-spectrum}]
The feasible sets in Definition~\ref{def:support-gap} contain zero, are
closed intervals, and are bounded above by the corresponding hidden-block
minimum eigenvalues. Thus the maxima exist.
For $\Sigma\succ0$, put $H:=\Sigma^{1/2}J\Sigma^{1/2}$.
Sylvester's inertia theorem gives $d_v$ negative and $d_u$ positive
eigenvalues. Since $H^{-1}=\Sigma^{-1/2}J\Sigma^{-1/2}$,
$\Sigma-cJ=\Sigma^{1/2}(\Id-cH^{-1})\Sigma^{1/2}$ gives
$\aplus{\Sigma}=\lambda_{d_v+1}(H)$ and
$\aminus{\Sigma}=-\lambda_{d_v}(H)$.
For singular $\Sigma$, apply this to $\Sigma+\epsilon\Id$.
As $\epsilon\downarrow0$, each margin decreases to a finite limit that
is feasible for $\Sigma$ by closedness, and hence equals its margin.
Continuity of the square root and of ordered symmetric eigenvalues
proves \eqref{eq:gap-spectrum} without deleting any zero eigenvalues.
Finally $H$ and $J\Sigma$ have the same characteristic polynomial by
the $AB,BA$ identity, even when they are not similar.
\end{proof}

\begin{remark}[Sharpness of the supporting constants]
For every $\Sigma\succeq0$, testing the supporting inequalities on the two
coordinate subspaces gives the upper bounds in
\[
 \lambda_{\min}(\Sigma)\le\aplus{\Sigma}\le\lambda_{\min}(\Suu),\qquad
 \lambda_{\min}(\Sigma)\le\aminus{\Sigma}\le\lambda_{\min}(\Svv).
\]
The lower bounds follow from $\Id\mp J\succeq0$.
When $\Svu=0$, the two upper bounds are attained.
Consequently $2\lambda_{\min}(\Sigma)\le\gap{\Sigma}\le2\norm{\Sigma}_{\op}$.
Both margins are positive exactly when $\Sigma\succ0$: if both are
positive and $x\in\kerop\Sigma$, the supporting inequalities force
$x^\top Jx=0$ and then $(\Sigma-\aplus{\Sigma}J)x=0$, so $x=0$.
The convergence theorem requires only their sum to be positive, so it
also covers some singular covariances. Each positive margin makes the
corresponding principal block positive definite, giving the rank bounds
in Remark~\ref{rem:singular-support}.
\end{remark}

\subsection{Linear convergence}\label{app:gf}
We prove the spectral-margin theorem and its quantitative trajectory
bounds first. Conditional and determinant-based criteria follow as
supplementary comparisons.

\begin{proposition}[Quantitative trajectory bounds]\label{prop:gf-lift-tail}
Under Theorem~\ref{thm:gf}, set
$a_0:=\lambda_{\min}(\Sigma_0)$, $B_0:=\norm{\Sigma_0}_{\op}$,
$I_0:=\gap{\Sigma_0}^{-1}\sqrt{2\Delta_0/\kappa}$, and $B_\infty:=B_0e^{2I_0}$.
Then
\begin{align}
 a_0e^{-2I_0}\Id_m&\preceq\Sigma_t\preceq B_\infty\Id_m,\label{eq:gf-tube}\\
 \norm{\Sigma_t-\Sigma_\infty}_{\op}
 &\le2B_\infty I_0e^{-\kappa \gap{\Sigma_0}t},\label{eq:gf-tail}\\
 \norm{\Svu_t-\Svu_\infty}_{\F}
 &\le2B_\infty I_0e^{-\kappa \gap{\Sigma_0}t}.\label{eq:gf-predictor-tail}
\end{align}
For the characteristic lift of any initial law in $\cP_2$ with this
second moment, there is a limit $\mu_\infty$ and
\begin{equation}\label{eq:gf-w2}
 W_2(\mu_t,\mu_\infty)
 \le\min\!\left\{\sqrt{\tr\Sigma_0}\,e^{I_0},\sqrt{2B_\infty}\right\}
       I_0e^{-\kappa \gap{\Sigma_0} t}.
\end{equation}
The same bound applies to $\norm{W_t-W_\infty}_{\F}/\sqrt n$ for particles.
The smallest positive covariance eigenvalue also satisfies
$\lambda_{\min}^{+}(\Sigma_t)\ge
e^{-2I_0}\lambda_{\min}^{+}(\Sigma_0)$, including at $t=\infty$.
\end{proposition}

\begin{proof}[Proof of Theorem~\ref{thm:gf} and Proposition~\ref{prop:gf-lift-tail}]
The predictor equation \eqref{eq:block-flow-general} gives the exact
dissipation identity on every positive semidefinite trajectory:
\begin{equation}\label{eq:flow-dissipation}
\frac{\mathrm{d}}{\mathrm{d}t}\ell(\Svu_t)
=-\tr(E_t^\top E_t\Suu_t)-\tr(E_t^\top\Svv_tE_t)\le0.
\end{equation}
The transported hidden-block bounds from Proposition~\ref{lem:margins}
therefore give $-\dot\Delta\ge \gap{\Sigma_0}\norm E_{\F}^2\ge2\kappa \gap{\Sigma_0}\Delta$.
On any finite existence interval, the preceding dissipation and PL give,
where the loss gap is positive,
\[
 -\frac{\dd}{\dd t}\sqrt{\Delta_t}
 \ge\frac{\gap{\Sigma_0}\norm{E_t}_{\F}^2}{2\sqrt{\Delta_t}}
 \ge \gap{\Sigma_0}\sqrt{\kappa/2}\norm{E_t}_{\F}.
\]
At zero loss gap the flow is stationary. Hence, for every finite $T>t$,
\[
 \int_t^T\norm{E_r}_{\F}\,\dd r
 \le\frac1{\gap{\Sigma_0}}\sqrt{\frac2\kappa}
       (\sqrt{\Delta_t}-\sqrt{\Delta_T})
 \le I_0e^{-\kappa \gap{\Sigma_0}t}.
\]
Together with smoothness, this gives
\begin{equation}\label{eq:force-tail}
 \norm{G_t}_{\op}=\norm{E_t}_{\op}
 \le\sqrt{2L\Delta_0}e^{-\kappa \gap{\Sigma_0}t},\qquad
 \int_t^T\norm{G_r}_{\op}\,\dd r\le I_0e^{-\kappa \gap{\Sigma_0}t}.
\end{equation}
Both $\Phi_t$ and $\Phi_t^{-1}$ have norm at most $e^{I_0}$, even when
$\Sigma_0$ is singular. Their congruence representation proves
\eqref{eq:gf-tube} before the maximal existence time. The covariance vector
field is locally Lipschitz on all symmetric matrices, and the upper bound
prevents finite-time blow-up; hence the solution is global.
Integrating $\norm{\dot\Sigma_t}_{\op}\le2B_\infty\norm{G_t}_{\op}$
proves the Cauchy property and \eqref{eq:gf-tail}. Integrability of the
force also makes $\Phi_t$ Cauchy; its limit is invertible since
$\sigma_{\min}(\Phi_t)\ge e^{-I_0}$.
Thus $\Sigma_\infty=\Phi_\infty\Sigma_0\Phi_\infty^\top$ has the same
rank as $\Sigma_0$ and is optimal by \eqref{eq:gf-rate}.
The singular-value bound for congruence gives the claimed positive
eigenvalue floor, without requiring full rank.
Directly in predictor Frobenius norm,
$\norm{\dot\Svu_t}_{\F}\le2B_\infty\norm{E_t}_{\F}$;
its integral proves \eqref{eq:gf-predictor-tail} without a dimension factor.
Finally $\norm{\Phi_t-\Phi_\infty}_{\op}\le e^{I_0}I_0e^{-\kappa \gap{\Sigma_0}t}$.
Coupling the two characteristic images of the same initial particle gives
the first bound in \eqref{eq:gf-w2}. For the second, use
\[
 \left(\int\norm{G_t w}^2\,\dd\mu_t(w)\right)^{1/2}
 =\sqrt{\tr(G_t\Sigma_tG_t)}
 \le\sqrt{2B_\infty}\,\norm{E_t}_{\F},
\]
and integrate along the same characteristics. The identical estimate
holds for empirical factors after division by $\sqrt n$.
If $\Delta_0=0$, the first-order flow is stationary and
the same formulas hold with $I_0=0$.
\end{proof}
\paragraph{Alternative nondegeneracy criteria}
The same argument can use any independently established hidden-block
lower bounds. The common-bound formulation below uses $2c$ in place of
$\gap{\Sigma_0}$. For $\Sigma_0\succ0$, taking
$c:=\min(\aplus{\Sigma_0},\aminus{\Sigma_0})$ recovers a possibly weaker
rate than Theorem~\ref{thm:gf}.

\begin{lemma}[Conditional hidden-block criterion]
\label{thm:uniform-nondegeneracy-convergence}
Under Assumption~\ref{ass:loss}, start \eqref{eq:bures-G} from
$\Sigma_0\succ0$. If $\Suu_t\succeq c\Id$ and $\Svv_t\succeq c\Id$
for some $c>0$ throughout its maximal existence interval, the flow is
global and
\begin{equation}
\label{eq:uniform-nondegeneracy-energy-rate}
\Delta_t\le\Delta_0e^{-4\kappa ct}.
\end{equation}
Moreover $\Sigma_t$ converges to $\Sigma_\infty\succ0$ with
$\Svu_\infty\in\cS_\star$.
\end{lemma}

Quadratic growth further gives
$\dist(\Svu_t,\cS_\star)^2\le2\Delta_0e^{-4\kappa ct}/\lambda$,
where $\lambda=\kappa$ is always available.

\begin{proof}[Proof of Lemma~\ref{thm:uniform-nondegeneracy-convergence}]
The assumed hidden bounds give
$-\dot\Delta\ge2c\norm E_{\F}^2\ge4c\kappa\Delta$.
On every positive loss gap interval, differentiating $\sqrt{\Delta_t}$ gives
\[
 -\frac{\mathrm{d}}{\mathrm{d}t}\sqrt{\Delta_t}
 \ge c\sqrt{2\kappa}\,\norm{E_t}_{\F},\qquad
 \int_t^T\norm{G_r}_{\op}\,\mathrm{d}r
 \le\frac{\sqrt{\Delta_t}-\sqrt{\Delta_T}}{c\sqrt{2\kappa}}.
\]
The same bound holds through zero loss gap by stationarity and is valid
before the maximal existence time.
Both characteristic matrices $\Phi_t$ and $\Phi_t^{-1}$ are bounded by
the exponential of this integral. Their covariance congruence gives
\begin{equation}\label{eq:uniform-nondegeneracy-covariance-bounds}
 \norm{\Sigma_t}_{\op}\le\norm{\Sigma_0}_{\op}
  \exp\!\left(\frac1c\sqrt{\frac{2\Delta_0}{\kappa}}\right),\qquad
 \lambda_{\min}(\Sigma_t)\ge\lambda_{\min}(\Sigma_0)
  \exp\!\left(-\frac1c\sqrt{\frac{2\Delta_0}{\kappa}}\right).
\end{equation}
These bounds rule out finite-time blow-up. The tail of the integral proves convergence to a
positive-definite covariance with optimal predictor.
Quadratic growth gives the distance bound.
\end{proof}

For coercive losses, determinant conservation and the second spectral
law yield an alternative, generally much coarser, source of such bounds.
\begin{proposition}[Coercivity prevents covariance collapse]
\label{prop:coercive-nondegeneracy}
Under Assumption~\ref{ass:loss}, suppose $\ell$ is coercive and
$\Sigma_0\succ0$. Let $M:=\sup\{\norm S_{\F}:\ell(S)\le\ell(\Svu_0)\}$.
On its maximal existence interval the Bures flow satisfies
$c\Id_m\preceq\Sigma_t\preceq B\Id_m$, where
$B:=\sqrt{\tr((J\Sigma_0)^2)+2M^2}+M$ and
$c:=\det\Sigma_0/B^{m-1}>0$.
In particular both hidden blocks are bounded below by $c\Id$.
\end{proposition}

\begin{proof}[Proof of Proposition~\ref{prop:coercive-nondegeneracy}]
The characteristic congruence \eqref{eq:orbit} preserves positive
definiteness on the maximal existence interval. Loss dissipation gives
$\norm{\Svu_t}_{\F}\le M$. Since $\tr G_t=0$, Jacobi's formula yields
$\det\Sigma_t=\det\Sigma_0>0$. The second spectral law gives
\[
 \norm{\Suu_t}_{\F}^2+\norm{\Svv_t}_{\F}^2
 =\tr((J\Sigma_0)^2)+2\norm{\Svu_t}_{\F}^2
 \le\tr((J\Sigma_0)^2)+2M^2.
\]
Split $\Sigma_t$ into its block-diagonal and off-diagonal parts. Their
operator norms are bounded by the square root of the last expression
and by $\norm{\Svu_t}_{\op}\le M$, respectively. Thus
$\norm{\Sigma_t}_{\op}\le B$. Finally,
$\det\Sigma_0\le\lambda_{\min}(\Sigma_t)B^{m-1}$ gives
$\Sigma_t\succeq c\Id_m$, and the same lower bound holds on both
principal blocks.
\end{proof}

\begin{remark}[Coarse ellipticity versus the spectral-margin rate]
\label{rem:coarse-rate}
At $\Sigma_0=\alpha\Id_m$, the determinant argument gives
\[
c_{\rm coarse}
:=\alpha\left(\frac{\alpha}
{\sqrt{m\alpha^2+2M^2}+M}\right)^{m-1}.
\]
Its rate $4\kappa c_{\rm coarse}$ can be extremely small in dimension,
whereas the support gap $\gap{\Sigma_0}=2\alpha$ gives $4\alpha\kappa$ directly.
For general initial covariances, the margin argument likewise gives a
rate determined by the signed spectrum, with no additional dimension factor.
Trace and initial loss alone cannot provide a uniform positive rate as
the support gap tends to zero
(Example~\ref{ex:slow}).
\end{remark}

\begin{remark}[Degenerate boundary leaves]\label{rem:degenerate-boundary-leaves}
In the $2\times2$ covariance setting of Section~\ref{sec:laws}, write
$\Sigma=\begin{psmallmatrix}p&r\\r&q\end{psmallmatrix}$,
$b_0:=(p_0-q_0)/2$, and $\alpha^2:=p_0q_0-r_0^2$,
with predictor loss $\ell(r):=(r-\beta)^2/2$.
On the boundary \(b_0=\alpha=0\), the scalar equation is
\(\dot r=-2|r|(r-\beta)\).
If \(\beta\ne0\), the ray with \(r_0\beta>0\) converges to \(\beta\),
whereas \(r_0\beta<0\) converges to the nonoptimal apex \(r=0\);
the apex itself is stationary. For example, when \(\beta=1\) and
\(p_0=q_0=1\), \(r_0=-1\), one has
\(p_t=q_t=-r_t=(2e^{2t}-1)^{-1}\).
When \(\beta=0\), one has \(r_t=r_0/(1+2|r_0|t)\): convergence to the
optimal apex is algebraic for \(r_0\ne0\).
These boundary cases have $\gap{\Sigma_0}=0$ and are not covered by the
positive-support-gap convergence theorem.
\end{remark}

\begin{example}[Fixed trace and predictor, arbitrarily slow escape]\label{ex:slow}
Let $d_u=d_v=2$, $0<\delta<1$, and
\[
 \Sigma_0=\diag(\delta,2-\delta,\delta,2-\delta),\qquad
 \ell(S):=\tfrac12\norm{S-e_1e_1^\top}_{\F}^2.
\]
Here $L=\kappa=1$, $\tr\Sigma_0=4$, $\Svu_0=0$, and $\Delta_0=1/2$,
independently of $\delta$. The active scalar block evolves by
\[
 r'=2\sqrt{r^2+\delta^2}(1-r),\qquad r(0)=0.
\]
The instantaneous relative decay of the loss at zero is $4\delta$.
Until $r=1/2$, $r'\le2(r+\delta)$, so
$r(t)\le\delta(e^{2t}-1)$. The time to reach $1/2$ is at least
$\tfrac12\log(1+1/(2\delta))$, which diverges as $\delta\downarrow0$.
There can therefore be no uniform positive exponential rate with a fixed
finite prefactor determined only by the displayed trace, initial loss gap, $L$, and
$\kappa$. At $\delta=0$ the active block is stationary and nonoptimal.
Theorem~\ref{thm:gf} records exactly the missing scale: $\gap{\Sigma_0}=2\delta$.
The slow escape occurs even for a strongly convex loss with a unique
predictor minimizer.
\end{example}

\subsection{Finite-width convergence}\label{app:sampling}
Write $r_0:=\rankop\tSigma_0$ for the population covariance rank.
\begin{theorem}[Relative initialization stability]\label{thm:relative}
Under Assumption~\ref{ass:loss}, let $\tSigma_0\succeq0$ with
$\gap{\tSigma_0}>0$, and put $\widetilde B_0:=\norm{\tSigma_0}_{\op}$ and
$\widetilde\Delta_0:=\ell(\tSvu_0)-\ell_\star$.
If an empirical covariance satisfies
\begin{equation}\label{eq:relative}
 (1-\varepsilon)\tSigma_0\preceq\Sigma_0
 \preceq(1+\varepsilon)\tSigma_0,\qquad 0<\varepsilon<1,
\end{equation}
then its first-order factor flow is global, converges to an optimal
predictor, and its loss gap $\Delta_t=\ell(\Svu_t)-\ell_\star$ satisfies
\begin{equation}\label{eq:sample-rate}
 \Delta_t\le\Delta_0
 e^{-2(1-\varepsilon)\kappa \gap{\tSigma_0}t}.
\end{equation}
Only the lower inequality in \eqref{eq:relative} is needed for this claim.
With both inequalities, $d_\wedge:=\min(d_u,d_v)$ gives the deterministic
prefactor bound
\begin{equation}\label{eq:sample-prefactor}
 \Delta_0\le
 \left(\sqrt{\widetilde\Delta_0}+\varepsilon \widetilde B_0\sqrt{Ld_\wedge/2}\right)^2.
\end{equation}
For the sampling law in Corollary~\ref{cor:finite-width}, the width
condition \eqref{eq:width} guarantees \eqref{eq:relative} with probability
at least $1-\delta$.
\end{theorem}

\begin{proof}[Proof of Theorem~\ref{thm:relative}]
From $\tSigma_0-\aplus{\tSigma_0}J\succeq0$ and the lower relative inequality,
$\Sigma_0-(1-\varepsilon)\aplus{\tSigma_0}J\succeq0$, and similarly for
$\aminus{\tSigma_0}$. Transporting these inequalities along the empirical
fundamental matrix allows us to apply the proof of Theorem~\ref{thm:gf} with margin
sum $(1-\varepsilon)\gap{\tSigma_0}$ and initial empirical loss gap. The two-sided relative
comparison gives
$\norm{\Sigma_0-\tSigma_0}_{\op}\le\varepsilon \widetilde B_0$ and hence
$\norm{\Svu_0-\tSvu_0}_{\F}\le\sqrt{d_\wedge}\varepsilon \widetilde B_0$.
Equation~\eqref{eq:sample-prefactor} follows from \eqref{eq:perturb}.
For the probability bound, choose an orthonormal basis
$T\in\R^{m\times r_0}$ of the range of $\tSigma_0$ and write
$\tSigma_0=T\Lambda T^\top$, where $\Lambda\succ0$.
The vectors $\eta_i:=T^\top\xi_i\in\R^{r_0}$ have identity second
moment and subgaussian norm at most $K$, and
$w_i=T\Lambda^{1/2}\eta_i$.
The covariance tail bound~\citep[Exercise 4.7.3]{Vershynin2018}
applied in this range gives
\[
 \norm{n^{-1}\textstyle\sum_i\eta_i\eta_i^\top-\Id_{r_0}}_{\op}
 \le CK^2\left(\sqrt{\frac{r_0+u}{n}}+\frac{r_0+u}{n}\right)
\]
with failure probability at most $2e^{-u}$. Set $u:=\log(2/\delta)$ and
adjust $C$ in \eqref{eq:width}. Multiplication on the left by
$T\Lambda^{1/2}$ and on the right by its transpose gives
\eqref{eq:relative}.
\end{proof}
\begin{proof}[Proof of Corollary~\ref{cor:finite-width}]
Theorem~\ref{thm:relative} gives the relative covariance event and the
loss rate. Its lower bound gives
$\gap{\Sigma_0}\ge(1-\varepsilon)\gap{\tSigma_0}>0$,
so Theorem~\ref{thm:gf} gives global existence and factor convergence.
The two-sided event preserves the kernel of $\tSigma_0$, so the empirical
initial covariance is positive definite exactly when $\tSigma_0$ is.
\end{proof}

\paragraph{Small width under full-rank isotropic initialization}
The rank improvement in Remark~\ref{rem:rank-adaptive-width} concerns
the chosen population law. A positive population support gap alone does not
guarantee a positive empirical support gap at small width when
$\tSigma_0=\sigma^2\Id_m$, $\sigma>0$.
\begin{example}[Isotropic Gaussian sampling can lose the support gap]
\label{ex:white-small-width}
Let $d_u=D$, $d_v=1$, and sample from $N(0,\tSigma_0)$ with
$\tSigma_0=\sigma^2\Id_{D+1}$, $\sigma>0$.
There is a universal $c>0$ such that, if
$1\le n\le D/16$, then
$\mathbb P\{\gap{\Sigma_0}=0\}\ge1-4e^{-cD}$,
although $\gap{\tSigma_0}=2\sigma^2$.
Thus a width bound depending only on $\min(d_u,d_v)$ cannot uniformly
preserve a positive support gap when sampling from $N(0,\sigma^2\Id_m)$. Convergence of
these trajectories is not decided by this example.
\end{example}
\begin{proof}
Write $U_0=\sigma X$ and $V_0=\sigma y^\top$, where
$X\in\R^{D\times n}$ and $y\in\R^n$ have independent standard
Gaussian entries. Gaussian singular-value and norm concentration
\citep[Section~7.3]{Vershynin2018}, with deviation $\sqrt D/8$, gives
\[
 \sigma_{\min}(X)\ge\sqrt D-\sqrt n-\sqrt D/8\ge5\sqrt D/8,
 \qquad \norm y\le\sqrt n+\sqrt D/8\le3\sqrt D/8
\]
with probability at least $1-4e^{-cD}$, for a universal $c>0$.
Hence $X^\top X-yy^\top\succeq(D/4)\Id_n$.
The nonzero spectrum of $J\Sigma_0$ is the spectrum of
$\sigma^2(X^\top X-yy^\top)/n$, so it has $n$ positive eigenvalues
and $D+1-n\ge2$ zero eigenvalues. In \eqref{eq:gap-spectrum},
$d_v=1$ therefore gives $\gap{\Sigma_0}=\lambda_2-\lambda_1=0$.
\end{proof}

\paragraph{A separate initial-accuracy budget}
A relative covariance bound preserves the rate; an absolute bound also
controls the initial loss. In the sampling setting of
Corollary~\ref{cor:finite-width}, write
$\widetilde B_0:=\norm{\tSigma_0}_{\op}$, $\widetilde\Delta_0:=\ell(\tSvu_0)-\ell_\star$, and
$d_\wedge:=\min(d_u,d_v)$.
Besides $\varepsilon,\delta\in(0,1)$, choose an absolute tolerance
$\zeta>0$ and set
\begin{equation}\label{eq:sampling-budgets}
 \beta:=\min\{\varepsilon,\zeta/\widetilde B_0\},\qquad
 \overline\Delta:=
 \left(\sqrt{\widetilde\Delta_0}+\zeta\sqrt{Ld_\wedge/2}\right)^2.
\end{equation}

\begin{theorem}[High-probability finite-width convergence]
\label{thm:sampling-budgets}
Under Assumption~\ref{ass:loss}, let $\tSigma_0\succeq0$ with
$\gap{\tSigma_0}>0$ and use the sampling law and tolerances above.
There is a universal $C>0$ such that, if
\begin{equation}\label{eq:budget-width}
 n\ge CK^4\beta^{-2}\bigl(r_0+\log(2/\delta)\bigr),
\end{equation}
then with probability at least $1-\delta$ the empirical covariance has
positive support gap, the factor flow is global, and
\[
 \Delta_t\le\Delta_0e^{-2\kappa(1-\varepsilon)\gap{\tSigma_0}t}
 \le\overline\Delta e^{-2\kappa(1-\varepsilon)\gap{\tSigma_0}t}.
\]
The factors converge, their limiting predictor belongs to $\cS_\star$, and
$\dist(\Svu_t,\cS_\star)^2\le2\Delta_t/\kappa$.
\end{theorem}

\begin{proof}[Proof of Theorem~\ref{thm:sampling-budgets}]
Apply the relative concentration event with tolerance $\beta$.
It implies the relative bounds with tolerance $\varepsilon$ and
$\norm{\Sigma_0-\tSigma_0}_{\op}\le\beta \widetilde B_0\le\zeta$.
The two initial supporting inequalities give the margin bounds, and
Lemma~\ref{lem:pl} gives the loss gap bound. More precisely, putting
$\underline g:=(1-\varepsilon)\gap{\tSigma_0}$ and $\overline B_0:=(1+\varepsilon)\widetilde B_0$,
the same event gives
\begin{equation}\label{eq:sampling-event}
 a_\pm(\Sigma_0)\ge(1-\varepsilon)a_\pm(\tSigma_0),\quad
 \norm{\Sigma_0}_{\op}\le\overline B_0,\quad
 \Delta_0\le\overline\Delta.
\end{equation}
Theorem~\ref{thm:gf} then applies
to the empirical covariance. The distance estimate follows from
Lemma~\ref{lem:pl}. More generally, any valid quadratic-growth constant
$\lambda$ gives $\dist(\Svu_t,\cS_\star)^2\le2\Delta_t/\lambda$.
\end{proof}
\begin{proposition}[The Gaussian initial-error cost]\label{prop:gaussian}
For independent $u_i\sim N(0,\sigma^2\Id_{d_u})$ and
$v_i\sim N(0,\sigma^2\Id_{d_v})$, and the quadratic loss
$\ell(S):=\norm{S-\Svu_\star}_{\F}^2/2$, one has exactly
\begin{equation}\label{eq:gaussian-error}
 \mathbb E\Delta_0
 =\tfrac12\norm{\Svu_\star}_{\F}^2+\frac{\sigma^4d_ud_v}{2n}.
\end{equation}
\end{proposition}

For balanced dimensions and fixed $\sigma$, width $n$ proportional to $m$
therefore gives an expected initial sampling error of order $m$, despite
the dimension-free exponent $4(1-\varepsilon)\sigma^2\kappa$ at any
fixed relative tolerance $\varepsilon\in(0,1)$. Keeping that error bounded
in expectation requires $n$ of order $d_ud_v$, when the teacher norm is
fixed. Thus width proportional to $m$ can preserve the convergence rate
while allowing the expected initial error to grow with dimension.

\begin{proof}[Proof of Proposition~\ref{prop:gaussian}]
Each entry of $n^{-1}\sum_i v_i u_i^\top$ has mean zero and variance
$\sigma^4/n$. Sum these variances; the cross term with the teacher has
zero expectation.
\end{proof}

\subsection{Discrete-time convergence}\label{app:gd}
The factor update transports the initial supporting inequalities even
though it no longer preserves $J$ exactly. To track this error, define
\begin{equation}\label{eq:gd-signature}
\Phi_0:=\Id,\quad\Phi_{k+1}:=M_k\Phi_k,\quad
\mathcal J_k:=\Phi_kJ\Phi_k^\top,\quad\mathcal D_k:=\mathcal J_k-J.
\end{equation}
The cross block of \eqref{eq:gd-covariance} also gives
\begin{equation}\label{eq:gd-predictor-update}
 \Svu_{k+1}-\Svu_k
 =-\tau(E_k\Suu_k+\Svv_kE_k)+\tau^2E_k(\Svu_k)^\top E_k.
\end{equation}
The next two lemmas control the accumulated defect $\mathcal D_k$ and
the quadratic correction in the predictor update.

\begin{lemma}[Transported support and accumulated defect]\label{lem:gd-support}
If $\Sigma_0\succeq0$, with initial margins $\aplus{\Sigma_0},\aminus{\Sigma_0}$ as in
\eqref{eq:margins}, then every finite iterate satisfies
\[
 \Sigma_k-\aplus{\Sigma_0}\mathcal J_k\succeq0,\qquad
 \Sigma_k+\aminus{\Sigma_0}\mathcal J_k\succeq0.
\]
If $\norm{\mathcal D_k}_{\op}\le\eta$ for some $\eta\in[0,1)$, then
$\Suu_k\succeq(1-\eta)\aplus{\Sigma_0}\Id$ and
$\Svv_k\succeq(1-\eta)\aminus{\Sigma_0}\Id$. With $e_k:=\norm{E_k}_{\F}$,
\begin{equation}\label{eq:gd-defect-recursion}
 \mathcal D_{k+1}=M_k\mathcal D_kM_k+\tau^2G_kJG_k,\qquad
 \norm{\mathcal D_N}_{\op}
 \le\tau\exp\left(2\sum_{j<N}\tau e_j\right)
             \sum_{j<N}\tau e_j^2.
\end{equation}
\end{lemma}

\begin{proof}[Proof of Lemma~\ref{lem:gd-support}]
The proof of Proposition~\ref{lem:margins} gives the initial supporting
inequalities. Congruence by $\Phi_k$ transports them, regardless of
invertibility. The $uu$ and $vv$ blocks of $\mathcal J_k$ are respectively
at least $(1-\eta)\Id$ and at most $-(1-\eta)\Id$, proving the hidden
lower bounds. Since $G_kJ+JG_k=0$, one has
$M_kJM_k=J+\tau^2G_kJG_k$, giving the recursion. Its norm is at most
$(1+\tau e_k)^2\norm{\mathcal D_k}_{\op}+\tau^2e_k^2$.
Unroll from zero and bound the amplification products by
$\exp(2\sum_{j<N}\tau e_j)$.
\end{proof}
\paragraph{Descent and length}
The bootstrap needs two estimates: loss descent and a finite total
gradient length. The next lemma obtains both from the same step condition.
For the proof, abbreviate the descent scale, initial gradient bound, and
covariance bound as
\begin{equation}\label{eq:gd-constants}
 \alpha_\star:=\frac{\gap{\Sigma_0}}{2\sqrt2},\qquad
 M_{\nabla}:=\sqrt{2L\Delta_0},\qquad
 B_\star:=\norm{\Sigma_0}_{\op}e^{2I_\star}.
\end{equation}
Then $I_\star=\sqrt{2\Delta_0/\kappa}/\alpha_\star$ and
$q=1-2\alpha_\star\kappa\tau$. Condition~\eqref{eq:gd-step} is
equivalent to requiring both
$\tau\le\alpha_\star/(B_\star M_{\nabla}+\frac92LB_\star^2)$ and
$\tau\le(1-1/\sqrt2)\alpha_\star/(e^{2I_\star}\Delta_0)$.
\begin{lemma}[Descent with unequal hidden lower bounds]\label{lem:gd-descent}
Suppose $\norm{\Sigma_k}_{\op}\le B$, $\Delta_k\le\Delta_0$,
$\Suu_k\succeq b_+\Id$, and $\Svv_k\succeq b_-\Id$.
Write $h:=b_++b_->c>0$. If $\tau M_{\nabla}\le1$ and
$\tau(BM_{\nabla}+\frac92LB^2)\le h-c$, then
\begin{equation}\label{eq:gd-descent}
 \Delta_{k+1}\le\Delta_k-c\tau e_k^2
 \le(1-2c\kappa\tau)\Delta_k.
\end{equation}
If this descent holds for $j=r,\ldots,N-1$, then
\begin{equation}\label{eq:gd-length}
 \sum_{j=r}^{N-1}\tau e_j^2\le\Delta_r/c,\qquad
 \sum_{j=r}^{N-1}\tau e_j
 \le\frac1c\sqrt{\frac2\kappa}
            (\sqrt{\Delta_r}-\sqrt{\Delta_N}).
\end{equation}
\end{lemma}

\begin{proof}[Proof of Lemma~\ref{lem:gd-descent}]
Let $Z:=\Svu_{k+1}-\Svu_k$. By \eqref{eq:gd-covariance} and smoothness,
$\norm Z_{\F}\le3\tau Be_k$ and
$\ip{E_k}{Z}\le-\tau h e_k^2+\tau^2BM_{\nabla}e_k^2$.
The descent lemma therefore gives the first assertion. Telescope it for
the squared-gradient sum. At positive loss gap, PL gives
\[
 \sqrt{\Delta_j}-\sqrt{\Delta_{j+1}}
 \ge\frac{c\tau e_j^2}{2\sqrt{\Delta_j}}
 \ge c\tau\sqrt{\kappa/2}\,e_j.
\]
At zero loss gap $e_j=0$. Summation proves the length bound.
\end{proof}

The proof below also establishes the following explicit tails in
Theorem~\ref{thm:gd}:
\begin{equation}\label{eq:gd-tails}
 \frac{\norm{W_k-W_\infty}_{\F}}{\sqrt n}
 \le\sqrt{2B_\star}\,I_\star q^{k/2},\qquad
 \norm{\Sigma_k-\Sigma_\infty}_{\op}
 \le2B_\star I_\star q^{k/2}.
\end{equation}

\begin{proof}[Proof of Theorem~\ref{thm:gd}]
Let $\eta:=1-1/\sqrt2$. Since $\gap{\Sigma_0}\le2B_0$ for
$B_0:=\norm{\Sigma_0}_{\op}$, one has
$\alpha_\star\le B_0/\sqrt2\le B_\star/\sqrt2$.
The first restriction in \eqref{eq:gd-step} implies
\[
 \tau M_{\nabla}<1/\sqrt2,\qquad
 0<2\alpha_\star\kappa\tau
 \le\frac{4\kappa\alpha_\star^2}{9LB_\star^2}\le\frac29.
\]
In particular $q\in[7/9,1)$.
Inductively suppose, through iterate $N$, that
$\Delta_k\le q^k\Delta_0$,
$\norm{\Sigma_k}_{\op}\le B_\star$, and
$\norm{\mathcal D_k}_{\op}\le\eta$, with descent proved for $k<N$.
These assertions start at $N=0$ since $\mathcal D_0=0$.
At $N$, Lemma~\ref{lem:gd-support} supplies hidden lower bounds with sum
$(1-\eta)\gap{\Sigma_0}=2\alpha_\star$. Lemma~\ref{lem:gd-descent} applies with
$c=\alpha_\star$, $h=2\alpha_\star$, and $B=B_\star$, proving the next
loss-gap estimate. Its finite-horizon sums now give
\[
 \sum_{j=0}^{N}\tau e_j\le I_\star,\qquad
 \sum_{j=0}^{N}\tau e_j^2\le\Delta_0/\alpha_\star.
\]
Thus $\norm{\Phi_{N+1}}_{\op}\le e^{I_\star}$ implies the next covariance
upper bound, and \eqref{eq:gd-defect-recursion} gives
$\norm{\mathcal D_{N+1}}_{\op}
\le\tau e^{2I_\star}\Delta_0/\alpha_\star\le\eta$.
This closes the induction.

Taking tails in \eqref{eq:gd-length} yields
$\sum_{j\ge k}\tau e_j\le I_\star q^{k/2}$.
Factor increments have norm
\[
 \norm{W_{j+1}-W_j}_{\F}/\sqrt n
=\tau\sqrt{\tr(G_j\Sigma_jG_j)}
\le\tau\sqrt{2B_\star}\,e_j.
\]
This proves factor convergence and its tail estimate. Similarly
$\norm{\Phi_k-\Phi_\infty}_{\op}\le e^{I_\star}I_\star q^{k/2}$
gives the covariance tail bound by subtracting the two congruences.
The limit predictor is optimal by continuity and the loss-gap bound.
Finally, if $r:=\tau M_{\nabla}<1$, every factor update is invertible and
\[
 \sigma_{\min}(\Phi_k)\ge\exp\left(-\frac{I_\star}{1-r}\right),\qquad
 \lambda_{\min}^{+}(\Sigma_k)\ge\lambda_{\min}^{+}(\Sigma_0)
 \exp\left(-\frac{2I_\star}{1-r}\right).
\]
Indeed $\sigma_{\min}(M_j)\ge1-\tau e_j>0$ and
$\log(1-x)\ge-x/(1-r)$ for $0\le x\le r$. The bound makes
$\Phi_\infty$ invertible, proving rank preservation also at the limit.
\end{proof}
The factor-increment estimate also applies to population discrete
pushforward updates. Coupling each particle with its limiting position
gives a Wasserstein tail at most $\sqrt{2B_\star}I_\star q^{k/2}$.

\paragraph{Isotropic initialization and local stability}
For an exact isotropic empirical second moment
$\Sigma_0=W_0W_0^\top/n=\sigma^2\Id_m$ and $\Delta_0=\ell(0)-\ell_\star>0$,
$I_\star=2\sqrt{\Delta_0}/(\sigma^2\sqrt\kappa)$ and the explicit
condition is
\[
 \tau\le\min\left\{
 \frac1{2e^{2I_\star}\sqrt{L\Delta_0}
                 +\frac9{\sqrt2}L\sigma^2e^{4I_\star}},\quad
 \frac{(\sqrt2-1)\sigma^2}{2e^{2I_\star}\Delta_0}
 \right\}.
\]
The contraction factor is $1-\sqrt2\kappa\sigma^2\tau$.
Here the empirical covariance is exactly $\sigma^2\Id_m$. For iid Gaussian
initialization, its fluctuations are handled by Corollary~\ref{cor:sampled-gd}.
As $\Delta_0\downarrow0$ with $\sigma,L,\kappa$ fixed, the sufficient
threshold tends to $\sqrt2/(9L\sigma^2)$.

\begin{remark}[Local stability and the global stepsize]\label{rem:gd-local}
Suppose $\ell$ is $C^2$ near a minimizer at zero, and linearize at
$\Sigma=\sigma^2\Id$. The predictor perturbation evolves to first order by
$\delta\Svu\mapsto(\Id-2\tau\sigma^2\nabla^2\ell(0))[\delta\Svu]$.
Thus $0<\tau<1/(L\sigma^2)$ contracts its nonzero-curvature directions;
kernel directions are neutral at this order. Theorem~\ref{thm:gd} additionally controls
covariance growth and the accumulated signature defect along the entire
trajectory, explaining its more restrictive global threshold.
\end{remark}

\paragraph{Adjustable rate budgets}
Smaller steps permit a smaller loss in the conserved conditioning,
bringing the discrete exponent closer to its continuous-time value.
\begin{proposition}[Adjustable rate budgets]\label{prop:gd-adjustable}
Under the loss and initialization assumptions of Theorem~\ref{thm:gd},
choose $0<c<h<\gap{\Sigma_0}$ and set
$I_c:=\sqrt{2\Delta_0/\kappa}/c$, $B_c:=B_0e^{2I_c}$.
The convergence and rank conclusions of Theorem~\ref{thm:gd} hold with
contraction factor $q_c:=1-2c\kappa\tau$, and the tail bounds
\eqref{eq:gd-tails} hold with $(I_\star,B_\star,q)$ replaced by
$(I_c,B_c,q_c)$, if
\[
 0<\tau\le\min\left\{\frac1{2M_{\nabla}},\
 \frac{h-c}{B_cM_{\nabla}+\frac92LB_c^2},\
 \frac{(1-h/\gap{\Sigma_0})c}{e^{2I_c}\Delta_0}\right\}.
\]
\end{proposition}

Taking $c=(1-\epsilon)\gap{\Sigma_0}$, $h=(1-\epsilon/2)\gap{\Sigma_0}$ for any
$\epsilon\in(0,1)$ certifies exponent $2\kappa(1-\epsilon)\gap{\Sigma_0}$ in time
$k\tau$, at an explicit sufficiently small step. This is arbitrarily
close to the flow exponent.

\begin{proof}[Proof of Proposition~\ref{prop:gd-adjustable}]
Use $\eta=1-h/\gap{\Sigma_0}$ in the preceding induction and
Lemma~\ref{lem:gd-descent} with $c,h$. The positivity guard follows from
$c(h-c)\le \gap{\Sigma_0}^2/4\le B_0^2$: $2c\kappa\tau\le4/9$.
The remaining estimates use $I_c,B_c$.
\end{proof}
\begin{corollary}[A stepsize fixed before sampling]\label{cor:sampled-gd}
Use the population parameters and width condition of
Theorem~\ref{thm:sampling-budgets}, with the bounds in
\eqref{eq:sampling-event}, and define before sampling
\[
 \overline\alpha_\star:=\frac{\underline g}{2\sqrt2},\quad
 \overline M_{\nabla}:=\sqrt{2L\overline\Delta},\quad
 \overline I_\star:=\frac1{\overline\alpha_\star}
                    \sqrt{\frac{2\overline\Delta}{\kappa}},\quad
 \overline B_\star:=\overline B_0e^{2\overline I_\star}.
\]
Choose a deterministic positive stepsize satisfying
\begin{equation}\label{eq:sampled-gd-step}
 \tau\le\min\left\{
 \frac{\overline\alpha_\star}
 {\overline B_\star\overline M_{\nabla}+\frac92L\overline B_\star^2},\quad
 \frac{(1-1/\sqrt2)\overline\alpha_\star}
 {e^{2\overline I_\star}\overline\Delta}\right\}.
\end{equation}
With probability at least $1-\delta$, factor GD converges to an optimal
predictor, and
\[
 \Delta_k
 \le(1-2\overline\alpha_\star\kappa\tau)^k\Delta_0
 \le(1-2\overline\alpha_\star\kappa\tau)^k\overline\Delta.
\]
Its factor/covariance tail estimates hold with the corresponding barred
constants and contraction factor.
\end{corollary}

At fixed relative tolerance, the sampled stepsize can deteriorate with
the dimension through $\overline\Delta$. The independent absolute
accuracy budget controls this initial-loss contribution.

\begin{proof}[Proof of Corollary~\ref{cor:sampled-gd}]
On \eqref{eq:sampling-event}, the empirical descent scale is at least
$\overline\alpha_\star$, whereas the empirical loss gap, force bound, length
bound, and covariance upper bound are at most their barred counterparts.
Both admissible stepsize terms increase with the descent scale and
decrease with the loss gap and initial covariance norm. Hence the chosen
step satisfies Theorem~\ref{thm:gd} for the empirical initialization.
Its contraction factor is no larger than the stated one. If the realized
initial loss gap vanishes, the empirical factors are stationary, and the
conclusions hold directly.
\end{proof}

\paragraph{Boundary cases}
For the scalar loss $\ell(s):=(s+1)^2/2$, the singular initialization
$U_0=V_0\ne0$ preserves equality of the factors and keeps
$\Svu_k=\norm{U_k}^2/n\ge0$, so its loss gap is at least $1/2$.
Conversely, even full-rank isotropic initialization does not justify arbitrary steps:
for $\ell(s):=(s-1)^2/2$, $\Sigma_0=\Id_2$, and $\tau=1$, one gets
$\Sigma_1=\bigl(\begin{smallmatrix}2&2\\2&2\end{smallmatrix}\bigr)$
and $\Sigma_2=0$, after which training is trapped.

\subsection{Deep linear residual networks}\label{app:resnets}
Write
$P_{b:a}:=A_b\cdots A_a$ when $a\le b$, and $P_{b:a}:=\Id$ otherwise.
Thus $P_{>s}=P_{S-1:s+1}$ and $P_{<s}=P_{s-1:0}$.
Differentiating $P_S=P_{>s}A_sP_{<s}$ gives
\begin{equation}
  \label{eq:resnet-layer-gradient}
\nabla_{\Svu_s}\ell(P_S)
=h\widetilde E_s,
\qquad
\widetilde E_s=P_{>s}^\top E P_{<s}^\top.
\end{equation}
In residual mean-field time the particle equations are
$(\dot u_{s,i},\dot v_{s,i})
=-(\widetilde E_s^\top v_{s,i},\widetilde E_su_{s,i})$.
The cross block of \eqref{eq:resnet-layer-bures} is therefore
$\dot\Svu_s=-\widetilde E_s\Suu_s-\Svv_s\widetilde E_s$.

In the finite-depth setting of Theorem~\ref{thm:resnet-finite}, put
$\overline{\sigma}_S:=(1+B/S)^{S-1}$ and
$I_{\rm res}:=\overline{\sigma}_S(\overline g_S\underline{\sigma}_S^2)^{-1}\sqrt{2\Delta_0/\kappa}$.
The proof gives the common layer-force tail
\begin{equation}\label{eq:resnet-length}
 \int_t^\infty\norm{\widetilde E_s(r)}_{\F}\,\dd r
 \le I_{\rm res}e^{-\kappa\overline g_S\underline{\sigma}_S^2t}.
\end{equation}

\begin{proof}[Proof of Theorem~\ref{thm:resnet-finite}]
Write $g_s:=\gap{\Sigma_s(0)}$ within this proof.
The endpoint derivative and layer covariance equations give
\[
 -\dot\Delta
 =\frac1S\sum_{s=0}^{S-1}
 \left[\tr(\widetilde E_s^\top \widetilde E_s\Suu_s)+\tr(\widetilde E_s\widetilde E_s^\top\Svv_s)\right]
 \ge\frac1S\sum_{s=0}^{S-1}g_s\norm{\widetilde E_s}_{\F}^2.
\]
Both the lower and upper singular-value bounds of the residual factors
multiply through the prefix and suffix, so
$\underline{\sigma}_S\norm E_{\F}\le\norm{\widetilde E_s}_{\F}\le \overline{\sigma}_S\norm E_{\F}$.
Therefore
\[
 -\dot\Delta\ge\overline g_S\underline{\sigma}_S^2\norm E_{\F}^2
 \ge2\kappa\overline g_S\underline{\sigma}_S^2\Delta.
\]
This proves the loss bound on each finite existence interval using the
average layer margin. The square-root PL
argument used in Theorem~\ref{thm:gf}, followed by the upper force bound,
gives \eqref{eq:resnet-length} first with any finite endpoint.
Thus all layer fundamental matrices and their inverses have norm at most
$e^{I_{\rm res}}$. Their covariances remain bounded and positive definite.
The finite-dimensional vector field is locally Lipschitz; these bounds
prevent a finite maximal existence time. The force-tail bound now holds
on $[t,\infty)$ and implies layer convergence by the argument of
Theorem~\ref{thm:gf}. Continuity of the finite residual product proves
optimality of its limit.
\end{proof}

\begin{proof}[Proof of Theorem~\ref{thm:resnet-continuous}]
The $L^1$ training equation in Section~\ref{sec:resnets} means
\begin{equation}\label{eq:resnet-integral}
 \Sigma_s(t)=\Sigma_s(0)-\int_0^t
 (\widetilde G_s(r)\Sigma_s(r)+\Sigma_s(r)\widetilde G_s(r))\,\dd r,
 \qquad\widetilde G_s=G_{\widetilde E_s}.
\end{equation}
This equality in $L^1$ gives an absolutely continuous representative for
almost every layer. The path bound makes the forces uniformly bounded
over depth on finite training intervals, justifying differentiation
of the endpoint and integration over depth.
An $L^1$ residual field has invertible propagators. Gronwall applied to
each propagator and its inverse gives
$e^{-B}\norm E_{\F}\le\norm{\widetilde E_s}_{\F}\le e^B\norm E_{\F}$.
The depth integrals of the conserved supporting inequalities yield
\[
 -\dot\Delta
 =\int_0^1\left[\tr(\widetilde E_s^\top \widetilde E_s\Suu_s)
                    +\tr(\widetilde E_s\widetilde E_s^\top\Svv_s)\right]\,\dd s
 \ge\overline g e^{-2B}\norm E_{\F}^2.
\]
These integrals are finite by the $L^1$ covariance assumption and the
uniform-in-depth force bounds. PL proves \eqref{eq:continuous-resnet-rate}.
As $\gap{\Sigma_s(0)}>0$ almost everywhere and
$\gap{\Sigma_s(0)}\le2\norm{\Sigma_s(0)}_{\op}$, its integral is finite and positive.
The square-root argument gives the common layer-force tail bound
\[
 \int_t^\infty\norm{\widetilde E_s(r)}_{\F}\,\dd r
 \le I_{\rm res}e^{-\kappa\overline g e^{-2B}t},\qquad
 I_{\rm res}:=\frac{e^{3B}}{\overline g}
                         \sqrt{\frac{2\Delta_0}{\kappa}}.
\]
The fundamental-matrix argument proves almost-everywhere convergence
and positive definiteness. The covariance tails are bounded by
$2e^{2I_{\rm res}}\norm{\Sigma_s(0)}_{\op}I_{\rm res}
e^{-\kappa\overline g e^{-2B}t}$, an integrable bound over depth.
This proves $L^1$ convergence. Continuity of the endpoint map on bounded
$L^1$ sets, obtained by variation of constants and Gronwall, proves
optimality of the limiting endpoint.
\end{proof}

\begin{remark}[Sharpness of the boundedness hypothesis]
\label{rem:resnet-path-escape}
Unlike the two-layer flow of Theorem~\ref{thm:gf}, positive conserved
layer margins do not rule out unbounded weight growth as $S\to\infty$
or along continuous-depth training. A linear neural ODE has positive
determinant, so it cannot fit an orientation-reversing target; training
may instead approach a singular endpoint through unbounded residuals.
For a concrete example, take the invertible scalar target $p_\star:=-1$ and loss
$\ell(p):=(p+1)^2/2$, with $L=\kappa=1$. Initialize every layer with
$\Sigma_s(0)=\Id_2$, realizable at width $n=2$; every conserved support
gap equals $2$. At each finite depth $S\ge2$, all residuals decrease
to $-S$, so $\sup_{t\ge0}\max_s|\Svu_s(t)|=S$ and no uniform $B<S$
exists. In continuous depth they decrease without bound, giving
$\sup_{t\ge0}\int_0^1|\Svu_s(t)|\,\dd s=\infty$.
Both flows are global, but their endpoints tend to zero and their losses
to $1/2>\ell_\star=0$. At finite depth the target is representable:
failure occurs along this symmetric trajectory, whereas the neural ODE
excludes the target altogether. In both cases, positive conserved margins
therefore fail to ensure the residual bounds required by the theorems.
\end{remark}

\begin{proof}[Verification of Remark~\ref{rem:resnet-path-escape}]
The scalar target $-1$ is invertible, and
$|\ell'(p)|^2=(p+1)^2=2\ell(p)$, so $L=\kappa=1$ and
$\ell_\star=0$. The width-two initialization
$U_s(0)=(\sqrt2,0)$, $V_s(0)=(0,\sqrt2)$ gives
$\Sigma_s(0)=\Id_2$ and initial endpoint $1$ at every depth.
For scalar layers write $a_s:=\Suu_s=\Svv_s$ and
$r_s:=\Svu_s$. Equality of the diagonal blocks is preserved, and the
layer equations reduce to
\[
 \dot a_s=-2\widetilde E_s r_s,\qquad
 \dot r_s=-2\widetilde E_s a_s.
\]
Thus $a_s^2-r_s^2=1$, and positivity gives
\[
 \Sigma_s(t)=\begin{pmatrix}
  \sqrt{1+r_s(t)^2}&r_s(t)\\
  r_s(t)&\sqrt{1+r_s(t)^2}
 \end{pmatrix}.
\]
The eigenvalues of $J\Sigma_s$ are $\pm1$, whereas those of
$\Sigma_s$ are $\sqrt{1+r_s^2}\pm|r_s|>0$.
Uniqueness and symmetry preserve identical layers. Denote their common
residual by $r_S$. At finite depth,
$\widetilde E_s=(p_S+1)(1+r_S/S)^{S-1}$, giving
\begin{equation}\label{eq:resnet-escape-finite}
 \begin{aligned}
 p_S&=(1+r_S/S)^S,\qquad r_S(0)=0,\\
 \dot r_S&=-2\sqrt{1+r_S^2}\big((1+r_S/S)^S+1\big)(1+r_S/S)^{S-1}.
 \end{aligned}
\end{equation}
For $S\ge2$, its vector field is locally Lipschitz and vanishes at $r_S=-S$.
Starting at zero, the solution decreases and cannot reach this equilibrium
in finite time. It stays in the compact interval $[-S,0]$, so it is global.
Any limit strictly above $-S$ would have strictly negative velocity;
hence its limit is $-S$, $p_S\to0$, and $\ell(p_S)\to1/2$.
In particular, every candidate bound $B<S$ is exceeded at a finite
training time, although $|r_S(t)|<S$ at every finite time.

The finite-depth architecture can represent $-1$ by taking one residual
factor equal to $-1$ and the others equal to $1$. The obstruction here
is the preserved symmetry of the initialized trajectory: all its
residual factors stay positive and approach zero. At $S=1$, this
obstruction disappears; $\dot r_1=-2\sqrt{1+r_1^2}(r_1+2)$ instead
gives $r_1\to-2$ and the optimal endpoint $-1$.

For a depth-constant field, the continuous propagators are
$P_{b,a}=e^{(b-a)r}$, so $\widetilde E_s=(e^r+1)e^r$ and
\begin{equation}\label{eq:resnet-escape-continuous}
 p=e^r,\qquad \dot r=-2\sqrt{1+r^2}(e^r+1)e^r,\qquad r(0)=0.
\end{equation}
For $r\le0$ its velocity is bounded, strictly negative, and has no finite
zero. The solution is therefore global and tends to $-\infty$.
Since $\int_0^1|\Svu_s(t)|\,\dd s=|r(t)|$, every finite integral
bound is eventually exceeded as well. The endpoint tends to zero.
The limiting loss $1/2$ is the infimum over positive scalar endpoints
and exceeds the ambient minimum $\ell_\star=0$.
The associated depth-constant covariance is a global locally absolutely
continuous $L^1$ trajectory, positive definite at every finite time.
Moreover, for $S\ge2$ and $-S<r_S\le0$,
$(1+r_S/S)^{S-1}\le e^{r_S/2}$ and $(1+r_S/S)^S+1\le2$.
Since $\sqrt{1+r^2}e^{r/2}\le1$ for $r\le0$, both $|\dot r_S|$ and
$|\dot r|$ are at most $4$. On $[0,T]$, all residuals therefore lie in
$[-4T,0]$. The finite-depth vector fields converge uniformly on this
compact interval, and continuous dependence gives $r_S\to r$ uniformly
on every bounded training interval $[0,T]$.
\end{proof}

\begin{remark}[Zero-target variant]
The same path bounds can fail while the loss tends to zero. With
$\ell(p):=p^2/2$ and the same initialization, the equations are
$\dot r_S=-2\sqrt{1+r_S^2}(1+r_S/S)^{2S-1}$ and
$\dot r=-2\sqrt{1+r^2}e^{2r}$. The preceding argument gives
$r_S\to-S$ for every $S\ge1$, $r\to-\infty$ in continuous depth,
and endpoints tending to the zero target. A uniform residual-path
bound would keep the endpoints bounded away from zero instead.
\end{remark}

\begin{remark}[Residual-path control and orientation]
Positive-definite layer covariances can coexist with singular residual
products. For a scalar depth-two network, take both layer predictors equal to $-2$,
so $A_0=A_1=0$, and the endpoint loss $\ell(p):=(p-1)^2/2$.
Both backpropagated gradients vanish although the endpoint is not optimal.
These predictors admit positive-definite layer covariances
$\bigl(\begin{smallmatrix}\sqrt5&-2\\-2&\sqrt5\end{smallmatrix}\bigr)$,
whose signed eigenvalues are $+1,-1$. Extending the earlier discretization and sampling arguments to this setting
requires control of the residual products along the corresponding training
trajectory; the two-layer GD theorem does not apply verbatim.

In continuous depth, $\det P(t)=
\exp(\int_0^1\tr\Svu_s(t)\,\dd s)>0$.
The uniform path bound also keeps $\sigma_{\min}(P(t))\ge e^{-B}$,
so any convergent endpoint remains in the orientation-preserving
invertible component. If every global minimizer has negative determinant,
these path bounds are incompatible with convergence to a global minimizer.
\end{remark}

\subsection{Heavy-ball momentum}\label{app:momentum-proof}
We first expand the closed phase equation, then establish global
existence, position-deformation estimates, and the energy bootstrap.
For positions $x:=(u,v)$ and velocities $y:=(p,q)$, let
$\nu_t\in\cP_2(\R^{2m})$ be their joint distribution. Its covariance is
$\Phase_t:=\int\binom xy\binom xy^{\!\top}\,\dd\nu_t(x,y)$,
with mixed block $\Xi_t:=\int xy^\top\,\dd\nu_t(x,y)$.
Finite-width factors correspond to
$\nu_t:=n^{-1}\sum_i\delta_{(x_i(t),y_i(t))}$.
The characteristics satisfy
\begin{equation}
\label{eq:single-particle-momentum-main}
\dot x=y,\qquad
\dot y=-G(\Svu)x-\gamma y.
\end{equation}
Taking its second moments gives the block system
\begin{equation}
\label{eq:phase-block-ode}
\begin{aligned}
\dot\Spos_t&=\Xi_t+\Xi_t^\top,\qquad
\dot\Xi_t=\Svel_t-\Spos_tG_t-\gamma\Xi_t,\\
\dot\Svel_t&=-G_t\Xi_t-\Xi_t^\top G_t-2\gamma\Svel_t.
\end{aligned}
\end{equation}
In particular, the predictor velocity and acceleration are
\begin{equation}\label{eq:predictor-speed}
 \dot\Svu=\Sigma^{qu}+\Sigma^{vp},\qquad
 \ddot\Svu+\gamma\dot\Svu=2\Sqp-E\Suu-\Svv E.
\end{equation}

\begin{proof}[Proof of Proposition~\ref{prop:phase-laws}]
Since \(G_t\) is symmetric,
\(\Mob_t^\top \Jcan+\Jcan\Mob_t=-\gamma \Jcan\).
Differentiate the two-characteristic pairing and the fundamental-matrix
identity. Invertibility also gives
\(\Psi_t\Jcan(\Psi_t)^\top=e^{-\gamma t}\Jcan\), whence
\(\Jcan\Phase_t=e^{-\gamma t}
(\Psi_t)^{-\top}(\Jcan\Phase_0)(\Psi_t)^\top\).
\end{proof}

\begin{proposition}[Global full phase flow]\label{prop:global-phase}
If $\ell$ is $C^1$, has locally Lipschitz gradient, and is bounded below,
every $\Phase_0\succeq0$ has a global positive semidefinite solution of
\eqref{eq:phase-block-ode}.  This assertion alone gives neither uniform position ellipticity
nor convergence to a minimizer.
\end{proposition}

\begin{proof}[Proof of Proposition~\ref{prop:global-phase}]
Local well-posedness follows from local Lipschitz continuity. The linear
phase characteristic lift preserves positivity. The energy identity gives
$\tr\Svel_t\le2\cE_0$. Positivity and Cauchy--Schwarz yield
\[
 \sqrt{\tr\Spos_t}
 \le\sqrt{\tr\Spos_0}+t\sqrt{2\cE_0},\qquad
 \norm{\Xi_t}_{\F}^2\le\tr\Spos_t\tr\Svel_t.
\]
The first inequality also follows directly from the $L^2$ triangle
inequality along characteristics, including times when the trace is zero.
Every block is bounded on each finite interval; continuation proves
global existence.
\end{proof}
For the matrix calculations in the momentum sections, abbreviate
$A:=\Spos$ and $B:=\Svel$. Whenever $A\succ0$, define
\begin{equation}\label{eq:regression}
 D:=\Xi^\top A^{-1},\qquad Q:=B-DAD^\top\succeq0.
\end{equation}
Here $D$ is a regression matrix, not assumed symmetric, and $Q$ is its
positive semidefinite residual covariance (the Schur complement).
They encode all phase blocks by $\Xi=AD^\top$ and $B=DAD^\top+Q$.
\begin{proposition}[General covariance coordinates]\label{prop:regression}
On the chart $A\succ0$, \eqref{eq:phase-block-ode} is equivalent to
\begin{equation}\label{eq:regression-ode}
 \begin{aligned}
 \dot A&=DA+AD^\top,\\
 \dot D&=-G-\gamma D-D^2+QA^{-1},\\
 \dot Q&=-2\gamma Q-DQ-QD^\top.
 \end{aligned}
\end{equation}
Moreover, positive semidefiniteness controls the position deformation:
\begin{equation}\label{eq:regression-bound-main}
 \norm D_{\op}^2\le\frac{2\mathcal{K}}{\lambda_{\min}(A)}.
\end{equation}
\end{proposition}

\begin{proof}[Proof of Proposition~\ref{prop:regression}]
Differentiate $D=\Xi^\top A^{-1}$ and $Q=B-DAD^\top$, using
\eqref{eq:phase-block-ode}. This gives the three displayed identities.
Since $DAD^\top\preceq B$, its trace is at most $2\mathcal{K}$, whereas
$\lambda_{\min}(A)\norm D_{\op}^2\le\tr(DAD^\top)$.
\end{proof}
\begin{lemma}[Conditional phase-energy decay]\label{lem:conditional}
Under Assumption~\ref{ass:loss}, let a positive semidefinite solution of
\eqref{eq:phase-block-ode} satisfy $a\Id_m\preceq\Spos_t\preceq b\Id_m$ for
$0\le t\le T$, where $0<a\le b<\infty$. Put
$\mathfrak K:=4a\kappa$, $\mathfrak L:=4bL$, and
$\rho:=\mathfrak L+2\sqrt{2L\cE_0}$. Then
\begin{equation}\label{eq:conditional}
 \cE_t\le3\cE_0 e^{-\lambda t},\qquad
 \lambda:=\frac{a\kappa\gamma}{\rho+2\gamma^2},\qquad 0\le t\le T.
\end{equation}
\end{lemma}
\begin{proof}
Set $\cN:=\tr(G\Spos G)$, the mean squared particle force.
The hidden principal blocks are between $a\Id$ and $b\Id$, hence
\[
 2a\norm E_{\F}^2\le\cN\le2b\norm E_{\F}^2.
\]
Combining PL with the smooth-gradient upper bound in
Lemma~\ref{lem:pl}, and applying Cauchy--Schwarz to the phase law, gives
\begin{equation}\label{eq:conditional-force-bounds}
 \mathfrak K\Delta\le \cN\le\mathfrak L\Delta,\qquad
 |\dot\Delta|=|\tr(G\Xi)|\le\sqrt{2\mathcal{K} \cN}.
\end{equation}
All second derivatives below are understood almost everywhere. Lipschitz
continuity of $\nabla\ell$ implies that $E_t$ is locally absolutely
continuous and $\norm{\dot E_t}_{\F}\le L\norm{\dot\Svu_t}_{\F}$.
Differentiating \eqref{eq:predictor-speed} gives
\[
 \ddot\Delta=H-\cN-\gamma\dot\Delta,\qquad
 H:=\ip{\dot E}{\dot\Svu}+2\ip E{\Sqp}\le\rho\mathcal{K}.
\]
For the last bound, positive phase blocks give
$\norm{\Sigma^{qu}}_{\F}^2\le b\tr\Sigma^q$ and
$\norm{\Sigma^{vp}}_{\F}^2\le b\tr\Sigma^p$, hence
$\norm{\dot\Svu}_{\F}^2\le4b\mathcal{K}$. Also
$\norm{\Sqp}_{\F}\le\sqrt{\tr\Sigma^q\tr\Sigma^p}\le\mathcal{K}$ and
$\norm E_{\F}\le\sqrt{2L\cE_0}$.

Choose $\epsilon:=\gamma/[2(\rho+2\gamma^2)]$ and
$\mathcal V:=\cE+\epsilon\dot\Delta$. Since
\[
 \epsilon\le\frac1{4\sqrt{2\rho}}
 \le\frac1{\sqrt{2\mathfrak L}},\qquad
 |\dot\Delta|\le\sqrt{\mathfrak L/2}\,\cE,
\]
one has $\cE/2\le\mathcal V\le3\cE/2$.
The inequality $\gamma\sqrt{2\mathcal{K} \cN}\le \cN/4+2\gamma^2\mathcal{K}$ yields
\[
 \dot{\mathcal V}\le-\tfrac34\epsilon \cN
       -\bigl(2\gamma-\epsilon(\rho+2\gamma^2)\bigr)\mathcal{K}
 \le-\tfrac34\epsilon\mathfrak K\Delta-\gamma\mathcal{K}.
\]
Now $\tfrac34\epsilon\mathfrak K\le\gamma$, since $a\kappa\le bL$
and $\rho\ge4bL$. It follows that
$\dot{\mathcal V}\le-\tfrac23(\tfrac34\epsilon\mathfrak K)\mathcal V
=-\lambda\mathcal V$.
Integrating the a.e. inequality proves \eqref{eq:conditional}.

\end{proof}

Under Theorem~\ref{thm:momentum}, the position-covariance tail is
\begin{equation}\label{eq:momentum-tail}
 \norm{\Spos_t-\Spos_\infty}_{\F}
 \le\frac{4\sqrt{6b_s\cE_0}}{\lambda_s(\gamma)}
       e^{-\lambda_s(\gamma)t/2}.
\end{equation}
For the position factors $W$ or the position marginal of the characteristic
phase law, the normalized Frobenius or $W_2$ tail has prefactor
$2\sqrt{6\cE_0}/\lambda_s(\gamma)$ at the same exponent.

\begin{proof}[Proof of Theorem~\ref{thm:momentum}]
The full phase flow is global by Proposition~\ref{prop:global-phase}.
On the connected position-positive chart containing time zero, let
$h(t):=\int_0^t\norm{D_r}_{\op}\,\dd r$.
The equation $\dot A=DA+AD^\top$ yields
$A_t=Z_t A_0Z_t^\top$ with $\dot Z=DZ$, $Z_0=\Id_m$, and thus
\[
 a_0e^{-2h(t)}\Id_m\preceq A_t\preceq b_0e^{2h(t)}\Id_m.
\]
As long as $h\le s$, Lemma~\ref{lem:conditional} and
Proposition~\ref{prop:regression} imply
\[
 h(t)\le\int_0^t\sqrt{2\mathcal{K}_r/a_s}\,\dd r
 \le\frac{2}{\lambda_s}\sqrt{\frac{6\cE_0}{a_s}}
 =\mathfrak S_s(\gamma)<s.
\]
There is no first time at which $h=s$. Nor can the chart terminate earlier:
on each such interval $A\succeq a_s\Id$ and the globally defined full
phase blocks stay bounded on finite horizons. This proves the global tube
and energy estimate. Furthermore the positive-block factorization gives
$\norm\Xi_{\F}^2\le b_s\tr B\le6b_s\cE_0 e^{-\lambda_st}$.
Integrate $\dot A=\Xi+\Xi^\top$ to get \eqref{eq:momentum-tail}.
For particles or a characteristic phase law, the remaining $L^2$ position displacement is
at most $2\sqrt{6\cE_0}\,e^{-\lambda_st/2}/\lambda_s$.
The limit is therefore attained, with zero loss gap and full phase
covariance $\diag(A_\infty,0)$.
\end{proof}

If $\cE_0=0$, the nonnegative kinetic and potential energies both vanish,
the force vanishes, and the momentum trajectory is stationary. Positive
definiteness of the \emph{phase} covariance is never required; in particular
zero-speed initialization is covered.

\subsection{An exact decomposition of hidden kinetic energy}
Prescribing the position covariance and its derivative does not determine
the kinetic energy: velocity components that leave the covariance
unchanged can still contribute to it.

\begin{proposition}[Minimum-energy lift of a covariance velocity]\label{prop:kinetic}
Fix $A\succ0$ and a symmetric matrix $H$ as the desired value of $\dot A$.
Let $C=C^\top$ be the unique solution of $CA+AC=H$.
Every positive semidefinite phase covariance with this position moment and
derivative has regression coordinates
\[
 D=C+\Omega,\qquad \Omega A+A\Omega^\top=0,\qquad Q\succeq0,
\]
and its kinetic energy decomposes as
\begin{equation}\label{eq:kinetic-decomposition}
 \mathcal{K}=\tfrac12\tr(CAC)
    +\tfrac12\tr(\Omega A\Omega^\top)+\tfrac12\tr Q.
\end{equation}
Thus the graph lift $D=C$, $Q=0$ has minimum kinetic energy.
An additional positive semidefinite Schur complement contributes exactly
$\tr Q/2$.
\end{proposition}
\begin{proof}
In an eigenbasis of $A$, the Sylvester equation divides $H_{ij}$ by the
strictly positive sum of eigenvalues, proving existence and uniqueness.
Equation $DA+AD^\top=H$ is equivalent to the stated condition on $\Omega$.
Since $B=DAD^\top+Q$, expand its trace. The two cross terms vanish in sum:
$\Omega A$ is skew-symmetric and $C$ is symmetric. Each term in
\eqref{eq:kinetic-decomposition} is nonnegative.
\end{proof}

\begin{example}[Identical position covariance and covariance derivative, arbitrarily large energy]
Let $A=\sigma^2\Id_m$, $C=0$, choose any skew-symmetric $\Omega$, and set
$D=\Omega$, $Q=0$. The initial predictor is zero because
$A=\sigma^2\Id_m$, and $\dot A=0$ makes its derivative zero as well,
independently of $\Omega$. Nevertheless
$\mathcal{K}=\sigma^2\norm\Omega_{\F}^2/2$ can be arbitrarily large. The phase
covariance is positive semidefinite; nonzero $\Omega$ gives a nonsymmetric regression matrix.
For fixed $a_0,b_0,L,\kappa$, the best feasibility ratio
\eqref{eq:budget} tends to zero as $\cE_0\to\infty$. Thus the initial
position covariance and predictor speed do not determine whether the
energy-damping certificate holds.
\end{example}



\subsection{Tightened conditional momentum rates}
\label{app:tightened-conditional-momentum}\label{app:momentum-details}

This subsection sharpens Lemma~\ref{lem:conditional}. It gives the exact
optimizer within the same Lyapunov estimates, the simpler parameter
choice used in the numerical figures, and the corresponding global
bootstrap with its energy prefactor.

\begin{proposition}[Tightened conditional energy decay]
\label{prop:tightened-conditional-momentum-pl}\label{prop:tightened}
Under the loss and position bounds of Lemma~\ref{lem:conditional} on
$[0,T]$, fix $\eta\in(0,1)$ and $\varepsilon>0$. With its constants
$\mathfrak K,\mathfrak L,\rho$, set
\begin{align}
 q&:=\varepsilon\sqrt{\mathfrak L/2},\label{eq:tightened-momentum-q-definition}\\
 r_{\varepsilon,\eta}
 &:=\min\{\varepsilon(1-\eta)\mathfrak K,\,
          2\gamma-\varepsilon(\rho+\gamma^2/(2\eta))\},\label{eq:tightened-momentum-dissipation-definition}\\
 \lambda_{\varepsilon,\eta}
 &:=r_{\varepsilon,\eta}/(1+q).\label{eq:tightened-momentum-rate-definition}
\end{align}
If $q<1$ and $r_{\varepsilon,\eta}>0$, the corrected energy
$\mathcal V_t:=\cE_t+\varepsilon\dot\Delta_t$ satisfies, for $0\le t\le T$,
\begin{align}
 (1-q)\cE_t&\le\mathcal V_t\le(1+q)\cE_t,\label{eq:tightened-momentum-lyapunov-equivalence}\\
 \mathcal V_t&\le\mathcal V_0e^{-\lambda_{\varepsilon,\eta}t},\label{eq:tightened-conditional-momentum-rate}\\
 \ell(\Svu_t)-\ell_\star
 &\le\cE_t\le\frac{\mathcal V_0}{1-q}e^{-\lambda_{\varepsilon,\eta}t}
 \le\frac{1+q}{1-q}\cE_0e^{-\lambda_{\varepsilon,\eta}t}.\label{eq:tightened-conditional-momentum-loss-rate}
\end{align}
If $\dot\Delta_0=0$, the first energy prefactor is $\cE_0/(1-q)$.
\end{proposition}

\begin{proof}
The force bounds \eqref{eq:conditional-force-bounds} give
$|\dot\Delta|\le\sqrt{2\mathfrak L\Delta\mathcal{K}}
\le\sqrt{\mathfrak L/2}\,\cE$, proving the equivalence.
For the derivative, use $\dot{\cE}=-2\gamma\mathcal{K}$,
$\ddot\Delta\le\rho\mathcal{K}-\mathcal N-\gamma\dot\Delta$, and Young's inequality:
\begin{equation}\label{eq:tightened-momentum-parameterized-young}
 \gamma|\dot\Delta|\le\gamma\sqrt{2\mathcal{K}\mathcal N}
 \le\eta\mathcal N+\frac{\gamma^2}{2\eta}\mathcal{K}.
\end{equation}
Since $\mathcal N\ge\mathfrak K\Delta$, this yields
\[
 \dot{\mathcal V}
 \le-\varepsilon(1-\eta)\mathcal N
     -[2\gamma-\varepsilon(\rho+\gamma^2/(2\eta))]\mathcal{K}
 \le-r_{\varepsilon,\eta}\cE
 \le-\lambda_{\varepsilon,\eta}\mathcal V.
\]
Integrate on $[0,T]$ and use the equivalence once more.
\end{proof}

\begin{proposition}[Optimal rate within the Lyapunov bounds]
\label{prop:optimal-tightened-conditional-momentum-rate}
Under Proposition~\ref{prop:tightened}, fix $\eta\in(0,1)$ and write
\begin{equation}\label{eq:tightened-momentum-A-B-definition}
 A_\eta:=(1-\eta)\mathfrak K,\qquad
 B_\eta:=\rho+\gamma^2/(2\eta).
\end{equation}
The unique maximizing correction at fixed $\eta$ is admissible and gives
\begin{align}
 \varepsilon_\eta^\star&=\frac{2\gamma}{A_\eta+B_\eta},\label{eq:tightened-momentum-optimal-epsilon}\\
 q_\eta^\star&:=\varepsilon_\eta^\star\sqrt{\mathfrak L/2}<1,\label{eq:tightened-momentum-optimal-q}\\
 \lambda_\eta&:=\frac{2\gamma A_\eta}
 {A_\eta+B_\eta+\gamma\sqrt{2\mathfrak L}}.\label{eq:tightened-momentum-fixed-eta-rate}
\end{align}
The unique maximizer over $\eta$ and its rate are
\begin{align}
 \eta_\star&=\frac{\gamma}
 {\gamma+\sqrt{\gamma^2+2\rho+2\gamma\sqrt{2\mathfrak L}}}
 \in(0,1/2),\label{eq:tightened-momentum-optimal-eta}\\
 \lambda_\star&:=\max_{0<\eta<1}\lambda_\eta
 =\frac{2\gamma\mathfrak K}
 {\mathfrak K+\frac12
 (\gamma+\sqrt{\gamma^2+2\rho+2\gamma\sqrt{2\mathfrak L}})^2}.\label{eq:tightened-momentum-optimal-rate}
\end{align}
\end{proposition}

\begin{proof}
At fixed $\eta$, the rate is the minimum of
$\varepsilon A_\eta/(1+\varepsilon\sqrt{\mathfrak L/2})$ and
$(2\gamma-\varepsilon B_\eta)/(1+\varepsilon\sqrt{\mathfrak L/2})$.
The first is strictly increasing and the second strictly decreasing,
so their intersection gives the unique optimum.
It is admissible because
$B_\eta>\rho+\gamma^2/2\ge\gamma\sqrt{2\rho}
\ge\gamma\sqrt{2\mathfrak L}$, hence $q_\eta^\star<1$;
both dissipation terms there equal $\varepsilon_\eta^\star A_\eta>0$.

For the remaining optimization, set
$\widetilde\rho:=\rho+\gamma\sqrt{2\mathfrak L}$ and divide numerator
and denominator of \eqref{eq:tightened-momentum-fixed-eta-rate} by $1-\eta$.
Maximizing the rate is equivalent to minimizing
$(\widetilde\rho+\gamma^2/(2\eta))/(1-\eta)$.
Its derivative has the sign of
$\widetilde\rho\eta^2+\gamma^2\eta-\gamma^2/2$, which increases strictly
and vanishes only at \eqref{eq:tightened-momentum-optimal-eta}.
The quotient there is
$\frac12(\gamma+\sqrt{\gamma^2+2\widetilde\rho})^2$, proving the formula.
\end{proof}

The numerical figures use the simpler admissible choice
\begin{equation}\label{eq:tightened-momentum-eta-approx}
 \eta_{\mathrm{approx}}
 :=\frac{\gamma}{\gamma+\sqrt{\gamma^2+2\rho}}\in(0,1/2).
\end{equation}
It minimizes $B_\eta/A_\eta$: differentiation gives
$\rho\eta^2+\gamma^2\eta-\gamma^2/2=0$.
Thus it optimizes the dissipation rate before accounting for the
equivalence factor $1+q_\eta^\star$. That factor is retained in the
following rigorous bound.

\begin{corollary}[Explicit rate used in the figures]
\label{cor:explicit-tightened-conditional-momentum-rate}\label{cor:tightened-choices}
Under Proposition~\ref{prop:tightened}, use
$\eta_{\mathrm{approx}}$ from \eqref{eq:tightened-momentum-eta-approx} and set
\begin{align}
 A_{\mathrm{approx}}&:=(1-\eta_{\mathrm{approx}})\mathfrak K,&
 B_{\mathrm{approx}}&:=\rho+\frac{\gamma^2}{2\eta_{\mathrm{approx}}},\label{eq:tightened-momentum-explicit-A-B}\\
 \varepsilon_{\mathrm{approx}}&:=\frac{2\gamma}{A_{\mathrm{approx}}+B_{\mathrm{approx}}},&
 q_{\mathrm{approx}}&:=\varepsilon_{\mathrm{approx}}\sqrt{\mathfrak L/2}.\label{eq:tightened-momentum-explicit-epsilon-q}
\end{align}
These parameters are admissible. With
$\mathcal V_0=\cE_0+\varepsilon_{\mathrm{approx}}\dot\Delta_0$,
\begin{align}
 \lambda_{\mathrm{approx}}
 &:=\frac{\varepsilon_{\mathrm{approx}}A_{\mathrm{approx}}}
         {1+q_{\mathrm{approx}}},\label{eq:tightened-momentum-explicit-rate}\\
 \ell(\Svu_t)-\ell_\star
 &\le\cE_t\le\frac{\mathcal V_0}{1-q_{\mathrm{approx}}}
       e^{-\lambda_{\mathrm{approx}}t},\qquad 0\le t\le T.\label{eq:tightened-momentum-explicit-loss-rate}
\end{align}
If $\dot\Delta_0=0$, the prefactor is $\cE_0/(1-q_{\mathrm{approx}})$.
\end{corollary}

\begin{proof}
The fixed-$\eta$ optimum in
Proposition~\ref{prop:optimal-tightened-conditional-momentum-rate}
is admissible for every $\eta\in(0,1)$.
Apply Proposition~\ref{prop:tightened} with $\eta=\eta_{\mathrm{approx}}$.
\end{proof}

The approximate choice concerns only optimization: its bound is exact,
and $\lambda_{\mathrm{approx}}\le\lambda_\star$. Neither exponent is
claimed to be the optimal rate of the underlying training dynamics.

\begin{corollary}[Global bootstrap with the correct prefactor]
\label{cor:tightened-bootstrap}
Under the loss and initialization hypotheses of Theorem~\ref{thm:momentum},
choose a tube $a_s:=a_0e^{-2s}$, $b_s:=b_0e^{2s}$ and an admissible
$(\varepsilon,\eta)$ for its constants. Put
\[
 \mathfrak C_0:=\frac{\cE_0+\varepsilon\dot\Delta_0}{(1-q)\cE_0}.
\]
If $\cE_0>0$ and
\begin{equation}\label{eq:tightened-bootstrap}
 \frac2{\lambda_{\varepsilon,\eta}}
 \sqrt{\frac{2\mathfrak C_0\cE_0}{a_s}}<s,
\end{equation}
then the tube and \eqref{eq:tightened-conditional-momentum-loss-rate} hold globally.
A sufficient condition using only initial position bounds and total
energy replaces $\mathfrak C_0$ by $(1+q)/(1-q)$.
\end{corollary}
\begin{proof}
Within the initial connected positive-position chart,
$\int_0^t\norm D_{\op}\le\int_0^t\sqrt{2\cE_r/a_s}\,\dd r$.
Proposition~\ref{prop:tightened} bounds this by the left side of
\eqref{eq:tightened-bootstrap}. The strict first-exit proof of
Theorem~\ref{thm:momentum} then applies unchanged. Initial positivity and
the Lyapunov equivalence imply $1\le \mathfrak C_0\le(1+q)/(1-q)$.
\end{proof}

\paragraph{The role of the prefactor}
A faster conditional exponent must be accompanied by its own prefactor
when testing the global deformation budget.
The sharper state-dependent prefactor uses the initial loss derivative
$\dot\Delta_0=\langle E_0,\Sigma^{qu}_0+\Sigma^{vp}_0\rangle$
as well as total energy. In particular, zero predictor speed gives
$\mathfrak C_0=1/(1-q)$. The parameter-only prefactor uses no extra initial derivative.
Maximizing the conditional exponent alone need not minimize the global
budget: the same choice of parameters also changes $q$ and hence the
prefactor. The closed-form optimizer should therefore be tested in
\eqref{eq:tightened-bootstrap}, not substituted there with an unchanged
prefactor.

\subsection{Large-position-scale certificate asymptotics}
\label{app:large-position-scale-certificate}\label{app:large-position-scale}
\label{subsec:large-position-scale-certificate}

Increasing position variance can open a damping window without reducing
the initial loss gap. We quantify that tradeoff and the best rates
certified by the same bootstrap.
For the two-scale white initialization of
Corollary~\ref{cor:two-scale-white-window}, the initial mechanical excess
energy is
\[
\cE_0
=
\ell(0)-\ell_\star+\frac{m\sigma_{qp}^2}{2}.
\]
This quantity is fixed when \(\sigma_{qp}\) is fixed and
\(\sigma_{vu}\to+\infty\).  It is useful to name the bootstrap gap
\[
\mathfrak G_s(\gamma):=\mathfrak S_s(\gamma)-s,
\qquad
s_{\mathrm{cert}}(\gamma)
:=
\inf\{s>0:\mathfrak G_s(\gamma)<0\}.
\]
Whenever this set is nonempty, define the corresponding boundary rate by
\[
\lambda_{\mathrm{cert}}(\gamma)
:=
\lambda_{s_{\mathrm{cert}}(\gamma)}(\gamma),
\]
using continuity in \(s\).  The strict bootstrap condition need not hold at
the boundary itself: the theorem applies at interior budgets
\(s>s_{\mathrm{cert}}\) with \(\mathfrak G_s(\gamma)<0\), whose rates approach
this boundary value.  Thus \(\lambda_{\mathrm{cert}}\) is a limiting supremal
certificate, rather than a rate asserted at a possibly nonadmissible boundary
point.
Define also the largest damping for which the bootstrap window is nonempty:
\[
\gamma^{\mathrm{cert}}_{\max}(\sigma_{vu})
:=
\sup\{\gamma>0:\text{ there exists }s>0
\text{ such that }\mathfrak G_s(\gamma)<0\}.
\]

The assumption \(\cE_0>0\) below excludes only the trivial zero-excess
case, in which both the predictor loss gap and kinetic energy vanish and the
trajectory is stationary.

\paragraph{Explicit damping interval and feasibility}
At a fixed deformation budget, admissibility reduces to a quadratic
inequality, making both damping boundaries explicit.
For \(\cE_0>0\) and fixed deformation budget \(s>0\), write
\begin{equation}\label{eq:window}
 c_s:=\frac{s\kappa a_s^{3/2}}{4\sqrt{3\cE_0/2}}.
\end{equation}
The theorem's strict certificate is exactly
\(2\gamma^2-c_s\gamma+\rho_s<0\), or
\begin{equation}\label{eq:window-roots}
 c_s^2>8\rho_s,\qquad
 \frac{c_s-\sqrt{c_s^2-8\rho_s}}4<\gamma<
 \frac{c_s+\sqrt{c_s^2-8\rho_s}}4.
\end{equation}
These are sufficient boundaries, not instability thresholds.

\paragraph{Optimizing feasibility}
To make this interval nonempty as easily as possible, maximize
\begin{equation}\label{eq:budget}
 \frac{c_s^2}{8\rho_s}
 =\frac{s^2\kappa^2a_0^3e^{-6s}}
 {192\cE_0(4b_0Le^{2s}+2\sqrt{2L\cE_0})}.
\end{equation}
Its unique maximizer $s_{\rm feas}$ belongs to $(1/4,1/3)$ and satisfies
\begin{equation}\label{eq:best-budget}
 \frac1{s_{\rm feas}}
 =3+\frac{4b_0Le^{2s_{\rm feas}}}
 {4b_0Le^{2s_{\rm feas}}+2\sqrt{2L\cE_0}}.
\end{equation}
For a proof, write the denominator term as $Ae^{2s}+B$, with $A,B>0$.
The logarithmic derivative is $2/s-6-2Ae^{2s}/(Ae^{2s}+B)$, and its
derivative is $-2/s^2-4ABe^{2s}/(Ae^{2s}+B)^2<0$.
It changes sign between $1/4$ and $1/3$. The resulting scalar equation
is explicit and uniquely solvable; no trajectory has to be computed.
This budget optimizes \emph{nonemptiness}, not the upper damping endpoint
or the decay exponent. Those are different optimization questions.

\paragraph{Asymptotic certificate bounds}
The large-position limit distinguishes fixed-damping behavior from the
largest damping that the certificate can accommodate.
\begin{proposition}[Large-position-scale certificate asymptotics]
\label{prop:large-position-scale-certificate}\label{prop:large-position}
For the two-scale white initialization above, fix \(\sigma_{qp}\ge0\), assume
\(\cE_0>0\), and
let \(\sigma_{vu}\to+\infty\).  For each fixed \(\gamma>0\) and fixed
\(s>0\),
\[
\mathfrak S_s(\gamma)
=
\frac{16Le^{5s}}{\kappa\gamma\,\sigma_{vu}}
\sqrt{\frac{3\cE_0}{2}}
+O(\sigma_{vu}^{-3}).
\]
The expansion is uniform for \(s\) in bounded sets.  Consequently,
\[
s_{\mathrm{cert}}(\gamma)
=
\frac{16L}{\kappa\gamma\,\sigma_{vu}}
\sqrt{\frac{3\cE_0}{2}}
+o(\sigma_{vu}^{-1}),
\]
and
\[
\lambda_{\mathrm{cert}}(\gamma)
=
\frac{\kappa\gamma}{4L}
\;-\;
\frac{16}{\sigma_{vu}}\sqrt{\frac{3\cE_0}{2}}
+o(\sigma_{vu}^{-1}).
\]
Moreover, the upper damping boundary satisfies
\[
\gamma^{\mathrm{cert}}_{\max}(\sigma_{vu})
=
\frac{\kappa\sigma_{vu}^3}
{24e\sqrt{3\cE_0/2}}
\bigl(1+o(1)\bigr).
\]
\end{proposition}

\begin{proof}
For fixed $\gamma>0$, the exact expression \eqref{eq:certificate}
can be expanded directly. Since
\[
\rho_s=4L\sigma_{vu}^2e^{2s}
+2\sqrt{2L\cE_0},
\]
we have
\[
\lambda_s(\gamma)
=
\frac{\kappa\gamma}{4Le^{4s}}
+O(\sigma_{vu}^{-2}).
\]
Substituting this into the definition of \(\mathfrak S_s(\gamma)\), with
\(a_s=\sigma_{vu}^2e^{-2s}\), gives the first expansion.
Write
\[
c_0:=\frac{16L}{\kappa\gamma}
\sqrt{\frac{3\cE_0}{2}}.
\]
The expansion is uniform near \(s=0\).  Thus every
\(s=c/\sigma_{vu}\), with \(c>0\) fixed, satisfies
\[
\mathfrak G_s(\gamma)
=\frac{c_0-c}{\sigma_{vu}}+o(\sigma_{vu}^{-1}),
\]
uniformly when \(c\) ranges over a compact set.  Hence budgets whose leading
constant is below \(c_0\) are inadmissible, whereas those whose leading
constant is strictly above \(c_0\) are admissible for all sufficiently large
\(\sigma_{vu}\).  This proves the displayed expansion for
\(s_{\mathrm{cert}}(\gamma)\).  Evaluating
\[
\lambda_s(\gamma)
=
\frac{\kappa\gamma}{4L}e^{-4s}
+O(\sigma_{vu}^{-2})
\]
at \(s=s_{\mathrm{cert}}(\gamma)\) gives the expansion for
\(\lambda_{\mathrm{cert}}(\gamma)\).

It remains to identify the high-damping boundary.  First observe the exact
upper bound
\[
\lambda_s(\gamma)
=\frac{a_s\kappa\gamma}{\rho_s+2\gamma^2}
\le\frac{\kappa\sigma_{vu}^2e^{-2s}}{2\gamma}.
\]
Consequently every admissible pair satisfies
\[
\mathfrak S_s(\gamma)
\ge
\frac{8\gamma e^{3s}}{\kappa\sigma_{vu}^3}
\sqrt{\frac{3\cE_0}{2}},
\]
and hence
\[
\gamma
<
\frac{\kappa\sigma_{vu}^3}{8\sqrt{3\cE_0/2}}\,s e^{-3s}
\le
\frac{\kappa\sigma_{vu}^3}{24e\sqrt{3\cE_0/2}}.
\]
This proves the required asymptotic upper bound uniformly over all budgets,
not only over bounded \(s\).

For the matching lower bound, take \(\gamma\) of order
\(\sigma_{vu}^3\) and keep \(s\) in a compact subset of \((0,+\infty)\).
Then $2\gamma^2$ dominates $\rho_s$ in \eqref{eq:certificate}, and
\[
\lambda_s(\gamma)
=
\frac{\kappa\sigma_{vu}^2e^{-2s}}{2\gamma}
\bigl(1+o(1)\bigr).
\]
Consequently,
\[
\mathfrak S_s(\gamma)
=
\frac{8\gamma e^{3s}}{\kappa\sigma_{vu}^3}
\sqrt{\frac{3\cE_0}{2}}
\bigl(1+o(1)\bigr),
\]
so the condition \(\mathfrak S_s(\gamma)<s\) is asymptotically equivalent to
\[
\gamma
<
\frac{\kappa\sigma_{vu}^3}{8\sqrt{3\cE_0/2}}\,
s e^{-3s}.
\]
The right-hand side is maximized at \(s=1/3\).  Choosing this budget and any
fixed multiplicative factor below the maximum gives an admissible pair for all
sufficiently large \(\sigma_{vu}\), proving the matching lower bound and the
claimed asymptotic formula.
\end{proof}

Proposition~\ref{prop:large-position-scale-certificate} has two complementary
messages.  At fixed damping, increasing the initial position scale makes the
certified ellipticity deformation vanish and the certified rate converges to
\(\kappa\gamma/(4L)\).  At fixed \(\sigma_{vu}\), however, this
particular bootstrap theorem cannot certify arbitrarily large damping: the
admissible region in \((s,\gamma)\) eventually disappears, with the upper
edge scaling like \(\sigma_{vu}^3\).  This is a limitation of the certificate,
not a claim that the underlying damped dynamics fail in the large-damping
regime.

\begin{figure}[t]
  \centering
  \def\mathdefault#1{#1}
  \subimport{figures/}{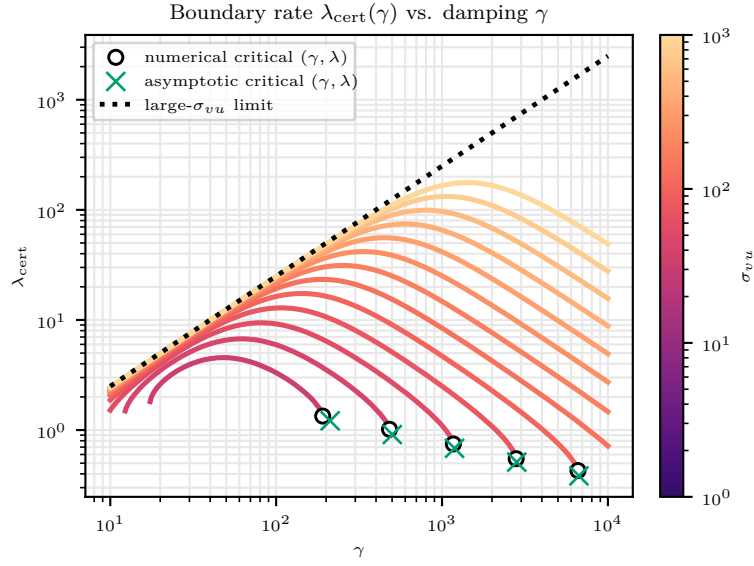}
\caption{
Limiting bootstrap rate at the admissibility boundary as a function of $\gamma$ for several initial position
scales $\sigma_{vu}$.
Solid curves are obtained by grid-searching the bootstrap certificate
for the boundary value \(\lambda_{s_{\mathrm{cert}}}(\gamma)\); rates at
strictly admissible nearby budgets approach these curves.  The dotted black curve is
the large-\(\sigma_{vu}\) limiting rate \(\kappa\gamma/(4L)\).  The
plot uses \((d_u, d_v) = (3, 2)\), \(\kappa=L=1\), \(\sigma_{qp}=1\), and
\(\Delta_0=\ell(0)-\ell_\star=1\).
For scales whose upper damping boundary lies within the scanned range, markers show numerical and asymptotic approximations to the maximal admissible damping and its corresponding boundary rate.
}
\label{fig:momentum-bootstrap-lambda-multi-sigma}\label{fig:momentum-boundary}
\end{figure}

\FloatBarrier


\clearpage

\section{Counting covariance-level mean-field laws}
\label{app:law-counting}

This appendix records the Lie-algebra computation behind the covariance-level
completeness statement used in Section~\ref{sec:laws}.
It is a continuum, covariance-level analogue of the local/generic finite-width
completeness theorem of \citet{MarcotteGribonvalPeyre2023}.

Width-consistent laws and their Wasserstein interpretation are defined
in Appendix~\ref{app:mean-field}. We work here directly with functions
of the covariance and the vector fields that must annihilate them.

\subsection{Split Bures vector fields}
\label{subsec:split-bures-fields}

To characterize the laws shared by the architecture, we allow every
smooth split loss:
\[
F(\Sigma):=\ell(\Svu),
\qquad
\Sigma=
\begin{pmatrix}
\Suu&\Suv\\
\Svu&\Svv
\end{pmatrix}.
\]
If \(E:=\nabla \ell(\Svu)\), set
\[
G_E:=
\begin{pmatrix}
0&E^\top\\
E&0
\end{pmatrix}\in\mathfrak p_J,
\qquad
\mathfrak p_J:=\{G\in\Ssym^m:GJ+JG=0\}.
\]
The corresponding Bures vector field is
\begin{equation}
\label{eq:appendix-split-vector-field}
X_\ell(\Sigma):=-(G_{\nabla \ell(\Svu)}\Sigma+\Sigma G_{\nabla \ell(\Svu)}).
\end{equation}
We denote the resulting space of split Bures vector fields by
\[
\mathcal X_{\rm split}
:=
\{X_\ell:\ell\in C^\infty(\R^{d_v\times d_u})\}.
\]
It is a real vector space.  Its constant-gradient subspace is obtained from
linear losses \(\ell_E(\Svu):=\ip{E}{\Svu}\):
\[
\mathcal X_{\rm lin}
:=
\{X_E:E\in\R^{d_v\times d_u}\},
\qquad
X_E(\Sigma):=-(G_E\Sigma+\Sigma G_E).
\]

\subsection{Generated Lie algebra}

Lie brackets reveal the directions generated by all split losses,
which determine how many independent common laws can remain.
Let
\[
\mathfrak o(J):=\{A\in\R^{m\times m}:A^\top J+JA=0\}
\]
be the Lie algebra of \(O(d_u,d_v)\).  For \(A\in\mathfrak o(J)\), define the
infinitesimal congruence field
\[
\xi_A(\Sigma):=-(A\Sigma+\Sigma A^\top).
\]
Its flow is \(\Sigma(t)=e^{-tA}\Sigma(0)e^{-tA^\top}\), so it is tangent to
the \(O(d_u,d_v)\)-congruence orbit of \(\Sigma\).  When \(A=G_E\in\mathfrak p_J\),
\(\xi_A=X_E\).

\begin{proposition}[Split fields generate \(\mathfrak o(d_u,d_v)\)]
\label{prop:appendix-lie-generated}
With the bracket convention \([X,Y]:=DY[X]-DX[Y]\), the map
\[
A\longmapsto \xi_A
\]
is a Lie-algebra homomorphism:
\[
[\xi_A,\xi_B]=\xi_{[A,B]}.
\]
Moreover,
\[
\operatorname{Lie}(\mathfrak p_J)=\mathfrak o(J).
\]
Consequently the Lie algebra generated by
\(\mathcal X_{\rm lin}\), the constant-gradient split Bures fields, is
\[
\{\xi_A:A\in\mathfrak o(J)\}.
\]
\end{proposition}

\begin{proof}
The homomorphism identity follows by differentiating the linear vector fields
\(\xi_A\).  It remains to compute the matrix Lie algebra generated by
\(\mathfrak p_J\).  For
\[
G_E=
\begin{pmatrix}
0&E^\top\\
E&0
\end{pmatrix},
\qquad
G_F=
\begin{pmatrix}
0&F^\top\\
F&0
\end{pmatrix},
\]
one has
\[
[G_E,G_F]
=
\begin{pmatrix}
E^\top F-F^\top E&0\\
0&EF^\top-FE^\top
\end{pmatrix}.
\]
These brackets span the block-diagonal skew-symmetric algebra
\[
\mathfrak k_J:=
\left\{
\begin{pmatrix}
A&0\\
0&D
\end{pmatrix}:A^\top=-A,\ D^\top=-D
\right\},
\]
by choosing \(E,F\) supported on one common row or one common column.  Since
\[
\mathfrak o(J)=\mathfrak k_J\oplus\mathfrak p_J,
\]
the generated Lie algebra is all of \(\mathfrak o(J)\).
\end{proof}

\paragraph{Fields versus their generated distribution}
The full space \(\mathcal X_{\rm split}\) contains nonlinear fields and
is not identified with a finite-dimensional matrix Lie algebra. Each
of its fields is a smooth, state-dependent linear combination of a
basis of the \(X_E\)'s. Consequently all such fields and their brackets
remain tangent to the same congruence orbits, while the constant-gradient
fields already generate the whole orbit tangent space. It is this
distribution that determines the common conservation laws.

\subsection{Completeness of the spectral laws}

The generated directions coincide with the tangent directions of
covariance congruence orbits; their codimension gives the local law count.
For a constant-gradient split Bures field, \(B:=J\Sigma\) satisfies
\[
\dot B=[G_E,B].
\]
The trace powers $\mathcal H_k$ in \eqref{eq:bures-G} are
therefore conserved. The next corollary states that, locally and generically, these
are the only smooth covariance-level laws common to all split losses.

Let
\[
\Omega_{\rm simp}
:=
\{\Sigma\in\Spp^m:J\Sigma\text{ has simple spectrum}\}.
\]

\begin{corollary}[Marcotte-style covariance-level completeness]
\label{cor:appendix-covariance-completeness}
Let \(h\) be a smooth function on an open subset of \(\Omega_{\rm simp}\).
If \(h\) is conserved by every split Bures vector field \(X_\ell\) of
\eqref{eq:appendix-split-vector-field}, then \(h\) is locally a smooth
function of the eigenvalues of \(J\Sigma\), equivalently of
\[
\mathcal H_1(\Sigma),\ldots,\mathcal H_m(\Sigma).
\]
These \(m\) laws are functionally independent on \(\Omega_{\rm simp}\).
Thus the characteristic polynomial of \(J\Sigma\) gives a complete
local/generic list of covariance-level mean-field conservation laws for the
split two-layer linear architecture.
\end{corollary}

\begin{proof}
It is enough to use the linear split losses.  If \(X_Eh=0\) for every \(E\),
then \(h\) is also annihilated by all iterated Lie brackets of the \(X_E\)'s.
By Proposition~\ref{prop:appendix-lie-generated},
\[
\mathrm{d}h(\Sigma)[\xi_A(\Sigma)]=0
\qquad\text{for all }A\in\mathfrak o(J).
\]
Thus \(h\) is constant on the connected pieces of the
\(O(d_u,d_v)\)-congruence orbits.

The orbit invariants are exactly the eigenvalues of \(J\Sigma\) on
\(\Omega_{\rm simp}\).  Indeed,
\[
J(Q\Sigma Q^\top)=Q^{-\top}(J\Sigma)Q^\top
\qquad(Q^\top JQ=J),
\]
so the spectrum is preserved.  Conversely, for \(\Sigma\succ0\), the matrix
\(J\Sigma\) is similar to the symmetric matrix
\(\Sigma^{1/2}J\Sigma^{1/2}\), hence has real nonzero eigenvalues with
\(d_u\) positive and \(d_v\) negative signs.  On the simple-spectrum stratum
its \(J\)-orthogonal eigenbasis can be locally normalized to a
\(J\)-orthonormal basis, so equal locally ordered eigenvalues imply membership
in the same local \(O(d_u,d_v)\)-orbit.  Therefore any smooth orbit-constant
function factors locally through the eigenvalues.  Newton identities replace
the eigenvalues by \(\mathcal H_1,\ldots,\mathcal H_m\).
More explicitly, for the locally ordered eigenvalues \(\lambda_j\),
\[
 \frac{\partial\mathcal H_k}{\partial\lambda_j}
 =k\lambda_j^{k-1},\qquad 1\le k,j\le m.
\]
This scaled Vandermonde matrix is invertible at distinct eigenvalues.
The eigenvalues themselves vary independently by varying the diagonal
covariance in a fixed normalized eigenbasis. Hence the spectral map has
rank \(m\), and its local level sets have dimension
\(m(m+1)/2-m=m(m-1)/2\), exactly the orbit dimension.
\end{proof}

For empirical covariances \(\Sigma=WW^\top/n\) with
\(W=\binom{U}{V}\), the finite-width Gram imbalance is
\[
\Gamma:=V^\top V-U^\top U=-W^\top JW.
\]
The nonzero spectra of \(J\Sigma=JWW^\top/n\) and
\(-\Gamma/n=W^\top JW/n\) coincide.  Corollary~\ref{cor:appendix-covariance-completeness}
is therefore the covariance quotient of the Marcotte--Gribonval--Peyr\'e
completeness theorem: the full labeled finite-width conservation data include the matrix
\(\Gamma\), while the covariance-level law keeps only its spectrum.

\subsection{A ReLU contrast}

Homogeneity preserves a scalar balance beyond linear networks, but it
does not give the same covariance closure or spectral family.
For a two-layer ReLU or more generally one-homogeneous network, the characteristic
equations associated with a potential \(v^\top\psi(u)\), with \(\psi\) one-homogeneous
and possibly time-dependent,
have the form
\[
\dot u=-D\psi(u)^\top v,\qquad \dot v=-\psi(u).
\]
Euler's identity \(D\psi(u)u=\psi(u)\) gives
\[
\frac{\mathrm{d}}{\mathrm{d}t}\bigl(\norm u^2-\norm v^2\bigr)=0.
\]
For ReLU this calculation is interpreted almost everywhere along the
characteristics, with an activation-boundary convention satisfying Euler's
identity (for instance derivative zero at a zero preactivation).
Thus the full mean-field flow preserves the pushforward distribution of
\(\norm u^2-\norm v^2\), and its covariance-level shadow is the total imbalance
\[
\int\bigl(\norm u^2-\norm v^2\bigr)\,\mathrm{d}\mu(u,v).
\]
Unlike the linear split case above, this ReLU statement is not obtained from a
closed finite-dimensional Bures covariance equation; it is included only to
highlight how much smaller the covariance-level conservation-law problem can be
than the full particle-level one. No completeness claim for ReLU laws is
made here.


\clearpage
\section{Intrinsic dynamics}\label{app:intrinsic}
Special conserved spectral values express some covariance blocks in terms
of fewer variables. This appendix derives these exact reductions.

\subsection{A two-point signed spectrum: hyperbolic isotropy}
A special two-point spectrum determines the hidden covariance blocks
from the predictor and therefore closes a smaller equation.
\begin{definition}[Hyperbolic isotropy]\label{def:hyperbolic-isotropy}
For \(\alpha>0\), the hyperbolically isotropic (HI) leaf is
\begin{equation}\label{eq:intrinsic-hi}
 \mathcal M_\alpha:=\{\Sigma\succeq0:\Sigma J\Sigma=\alpha^2J\}.
\end{equation}
\end{definition}
The identity forces $\Sigma\succ0$ and is equivalent to
$(J\Sigma)^2=\alpha^2\Id_m$. The symmetric representative of $J\Sigma$
then has $d_u$ eigenvalues $+\alpha$ and $d_v$ eigenvalues $-\alpha$;
equivalently, $\mathcal H_k=\alpha^k(d_u+(-1)^kd_v)$ for $k=1,\ldots,m$.
Spectral conservation proves invariance of $\mathcal M_\alpha$ despite
the repeated eigenvalues. Directly, the derivative of $\Sigma J\Sigma$
is $-G(\Sigma J\Sigma)-(\Sigma J\Sigma)G$, with constant solution $\alpha^2J$.

\begin{proposition}[Predictor-only intrinsic equation]\label{prop:intrinsic-hi}
Every predictor \(\Svu\in\R^{d_v\times d_u}\) has a unique covariance
on \(\mathcal M_\alpha\), given by
\begin{equation}\label{eq:intrinsic-graph}
 \Sigma=
 \begin{pmatrix}
 ((\Svu)^\top\Svu+\alpha^2\Id_{d_u})^{1/2}&(\Svu)^\top\\
 \Svu&(\Svu(\Svu)^\top+\alpha^2\Id_{d_v})^{1/2}
 \end{pmatrix}.
\end{equation}
On this leaf, the Bures dynamics close on the predictor:
\begin{equation}\label{eq:intrinsic-hyperbolic-dynamics}
 \dot\Svu=-E((\Svu)^\top\Svu+\alpha^2\Id)^{1/2}
       -(\Svu(\Svu)^\top+\alpha^2\Id)^{1/2}E,
 \qquad E:=\nabla\ell(\Svu).
\end{equation}
For \(d_u=d_v=1\), this is
\(\dot\Svu=-2\sqrt{(\Svu)^2+\alpha^2}\,\ell'(\Svu)\).
\end{proposition}
\begin{proof}
Expanding \eqref{eq:intrinsic-hi} gives
\[
 (\Suu)^2-(\Svu)^\top\Svu=\alpha^2\Id,\quad
 (\Svv)^2-\Svu(\Svu)^\top=\alpha^2\Id,\quad
 \Svu\Suu=\Svv\Svu.
\]
Positivity selects the positive square roots. Conversely, an SVD proves
the intertwining identity and gives eigenvalues \(\alpha\) and
\(\sqrt{\sigma_i^2+\alpha^2}\pm\sigma_i>0\), proving membership in
\(\mathcal M_\alpha\). Substituting the diagonal blocks into
\eqref{eq:block-flow-general} gives the intrinsic equation.
\end{proof}
The same eigenvalue calculation gives the exact extremal envelopes
\[
 \lambda_{\max}(\Sigma)=\sqrt{\norm{\Svu}_{\op}^2+\alpha^2}
                         +\norm{\Svu}_{\op},\qquad
 \lambda_{\min}(\Sigma)=\frac{\alpha^2}{\lambda_{\max}(\Sigma)}.
\]
In the general convergence theorems this special spectral level has
\(\aplus{\Sigma_0}=\aminus{\Sigma_0}=\alpha\). Thus its loss exponent is \(4\alpha\kappa\),
and the GD contraction factor is \(1-\sqrt2\alpha\kappa\tau\).

\paragraph{Finite-width form}
The covariance constraint has an exact finite-width form that depends
only on the product $WW^\top$.
For \(\Sigma=WW^\top/n\), condition \eqref{eq:intrinsic-hi} is
\(WW^\top JWW^\top=n^2\alpha^2J\). In terms of \(U,V\), it reads
\[
 \begin{gathered}
 (UU^\top)^2-UV^\top VU^\top=n^2\alpha^2\Id_{d_u},\qquad
 (VV^\top)^2-VU^\top UV^\top=n^2\alpha^2\Id_{d_v},\\
 VU^\top UU^\top=VV^\top VU^\top.
 \end{gathered}
\]
Any covariance in \eqref{eq:intrinsic-graph} admits empirical factors at
every width \(n\ge m\), for example \(W=\sqrt n\,\Sigma^{1/2}Q_0\)
with \(Q_0Q_0^\top=\Id_m\).

\paragraph{Comparison with relaxed balance}
\citet[Theorem 3.8]{MarcottePeyreGribonval2026} derive intrinsic dynamics
under a different condition,
\(\Gamma_{\rm rb}:=(V^\top V-U^\top U)/n=\beta\Id_n\),
on a full-column-rank stratum. It implies
\[
 (\Suu)^2+\beta\Suu=(\Svu)^\top\Svu,\qquad
 (\Svv)^2-\beta\Svv=\Svu(\Svu)^\top.
\]
For \(\beta\ne0\), zero singular directions have ambiguous branches.
Where both factors have full column rank, the range projectors
\(\Pi_u,\Pi_v\) of \((\Svu)^\top,\Svu\) select
\[
 \begin{aligned}
 \Suu&=\Pi_u[-\tfrac\beta2\Id+
             ((\Svu)^\top\Svu+\tfrac{\beta^2}{4}\Id)^{1/2}],\\
 \Svv&=\Pi_v[\tfrac\beta2\Id+
             (\Svu(\Svu)^\top+\tfrac{\beta^2}{4}\Id)^{1/2}].
 \end{aligned}
\]
If \(n=d_u=d_v\) and both factors are invertible, the projectors are
identities and the predictor equation coincides with
\eqref{eq:intrinsic-hyperbolic-dynamics} for \(\alpha=|\beta|/2\).
The leaves remain distinct: for \(\beta\ne0\), relaxed balance is equivalent to
\(\Sigma J\Sigma=-\beta\Sigma\), \(\rankop\Sigma=n\).
Indeed a left inverse of \(W/\sqrt n\) proves the converse.
It forces \(n\le d_v\) if \(\beta>0\), and \(n\le d_u\) if \(\beta<0\);
positive HI instead has rank \(m\).
For \(\beta=0\), positivity selects the unshifted square roots without a
rank assumption, giving the boundary equation \(\alpha=0\), not a
positive-definite initialization.
Without the rank qualifications, relaxed balance does not ensure
predictor closure: \(n=1,d_u=d_v=2,U=0,V=e_1\) or \(e_2\) gives the
same predictor and \(\beta=1\), but different predictor velocities for
\(\ell(S):=\norm{S+\Id_2}_{\F}^2/2\).

\begin{remark}[Conserved values do not generally close the predictor]
For \(d_u=d_v=2\), take \(\Svu=0\), \(\Svv=\diag(3,4)\), and
\(\Suu=\diag(1,2)\) or \(\diag(2,1)\).
Both states have signed spectrum \(\{1,2,-3,-4\}\), hence the same
spectral laws. For \(\ell(S):=\norm{S+e_1e_1^\top}_{\F}^2/2\),
their predictor velocities are \(-4e_1e_1^\top\) and
\(-5e_1e_1^\top\).
The special two-point spectrum gives an algebraic graph unavailable
at these general spectral values.
\end{remark}

\subsection{A single phase spectral scale: conformal symplectic isotropy}
A special conformal spectral identity eliminates the velocity covariance
within a positive-position chart of the phase flow.
\begin{definition}[Conformal symplectic isotropy]\label{def:csi}
A phase covariance is conformally symplectically isotropic (CSI) at
scale \(\theta\ge0\) if
\begin{equation}\label{eq:intrinsic-csi}
 \Phase \Jcan\Phase=\theta^2\Jcan.
\end{equation}
\end{definition}
For $\theta>0$, this forces $\Phase\succ0$ and is equivalent to
$\Phase^{1/2}\Jcan\Phase^{1/2}$ having $m$ copies of each eigenvalue
$+i\theta,-i\theta$.
The renormalized trace powers in Proposition~\ref{prop:phase-laws}
are zero at odd orders and $2m(-1)^k\theta_0^{2k}$ at order $2k$.
Conformal transport preserves \eqref{eq:intrinsic-csi} with
$\theta_t=e^{-\gamma t}\theta_0$, including the singular boundary
$\theta_0=0$, which is specified algebraically, without diagonalizability.

\begin{proposition}[Reduced phase dynamics]\label{prop:intrinsic-csi}
On the chart \(A:=\Spos\succ0\), condition \eqref{eq:intrinsic-csi}
is equivalent to
\[
 D=C=C^\top,\qquad Q=\theta^2A^{-1},
 \qquad \Xi=AC,\quad \Svel=CAC+\theta^2A^{-1}.
\]
The reduced equations are
\begin{equation}\label{eq:csi-reduced}
 \dot A=CA+AC,\qquad
 \dot C=-G-\gamma C-C^2+\theta_t^2A^{-2},
 \qquad\dot\theta_t=-\gamma\theta_t.
\end{equation}
They use only the two symmetric matrices \(A,C\), with known scale
\(\theta_t\), instead of all blocks of the phase covariance.
\end{proposition}
\begin{proof}
Substitute \(\Xi=AD^\top\), \(\Svel=DAD^\top+Q\).
The upper-left block of \eqref{eq:intrinsic-csi} is
\(A(D-D^\top)A=0\), forcing symmetry.
The upper-right block then becomes \(AQ=\theta^2\Id\);
the other blocks follow.
Substitution in Proposition~\ref{prop:regression} gives
\eqref{eq:csi-reduced}. This calculation includes \(\theta=0\).
\end{proof}
White position and velocity blocks with \(\Xi_0=0\) give
\(\theta_0=\sigma_{vu}\sigma_{qp}\). Zero velocity gives
\(C_0=0,\theta_0=0\), allowing any \(A_0\succ0\).
Neither choice is needed for convergence: the general theorem controls
\(D\) by energy even when it is nonsymmetric and \(Q\) is independent.

\begin{remark}[What the reduced state still contains]
Even on CSI, the two cross blocks \(\Svu,\Sqp\) do not generally form a
closed state. With \(\theta=1,C=0\), choose \(A=\Id_2\) or \(2\Id_2\)
and \(\Svel=A^{-1}\). Both have
\(\Svu=\Sqp=\dot\Svu=0\), yet
\(\ddot\Svu=-2E\) or \(-4E\) for the same nonzero scalar loss gradient.
Thus the full position block remains part of the intrinsic variables.
\end{remark}

\clearpage
\section{A nonquadratic example: smoothed Lasso}
\label{app:smooth-lasso}

To illustrate the smooth PL loss class beyond quadratic objectives, we repeat the three experiments of
Figure~\ref{fig:bures-convergence-panel} with a smoothed entrywise $\ell_1$ penalty on the predictor:
\begin{equation}\label{eq:smooth-lasso-loss}
  \ell(S)=\frac12\norm{S-S_{\rm teach}}_{\F}^2
  +\lambda\sum_{i,j}\Bigl(\sqrt{S_{ij}^2+\delta^2}-\delta\Bigr),
  \qquad \lambda=0.1,\quad \delta=0.01.
\end{equation}
The Hessian is diagonal, with entries
$1+\lambda\delta^2/(S_{ij}^2+\delta^2)^{3/2}\in[1,1+\lambda/\delta]$.
Thus $\ell$ is globally $1$-strongly convex and satisfies Assumption~\ref{ass:loss}
with $\kappa=1$ and $L=11$.
The penalty acts on $S$, not separately on the factors, so the covariance closure and spectral conservation laws apply without modification.
Its unique minimizer $S_{\rm reg}$ is generally different from the teacher; we compute it entrywise from
\[
  (S_{\rm reg})_{ij}-(S_{\rm teach})_{ij}
  +\lambda\frac{(S_{\rm reg})_{ij}}{\sqrt{(S_{\rm reg})_{ij}^2+\delta^2}}=0.
\]
All loss gaps below are relative to $\ell_\star=\ell(S_{\rm reg})$.
Positive smoothing does not generally yield exact sparsity; the nonsmooth limit $\delta=0$ is outside Assumption~\ref{ass:loss}.

\begin{figure}[htbp]
  \centering
  \resizebox{\linewidth}{!}{\input{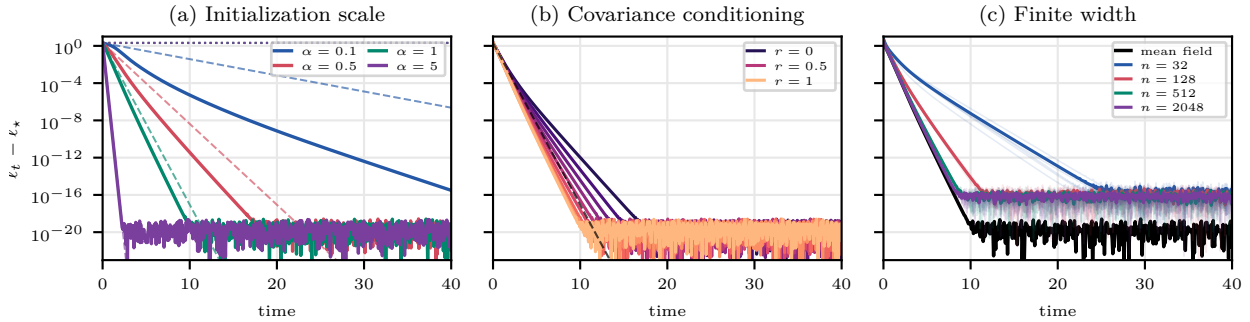}}

  \caption{Smoothed-Lasso counterpart of Figure~\ref{fig:bures-convergence-panel}, using the same dimensions $(d_u,d_v)=(16,9)$, teacher,
    initialization scales, and sampling seeds.
    \textbf{(a)} White covariance $\Sigma_0=\alpha\Id_m$ for $\alpha=0.1,0.5,1,5$.
    Dashed curves are the support-gap envelopes $\Delta_0 e^{-4\alpha\kappa t}$; dotted curves use the internal coarse determinant bound of Remark~\ref{rem:coarse-rate}.
    \textbf{(b)} One input-output diagonal pair is scaled by $r\in\{0,1/8,\ldots,1\}$.
    Here $g(\Sigma_0)=2r$; the dashed curve shows the white-reference envelope.
    \textbf{(c)} The population covariance flow from $\Id_m$ and finite-width flows at widths $n=32,128,512,2048$, with five independent trials each.}
  \label{fig:smooth-lasso-panel}
\end{figure}

\end{document}